\documentclass[10pt,twocolumn,twoside]{IEEEtran} 
\IEEEoverridecommandlockouts                       

\usepackage{fix-cm}
\usepackage{etex}

\usepackage{dblfloatfix}

\usepackage{nag}

\makeatletter
\@ifpackageloaded{xcolor}{}{%
\usepackage[table,x11names,dvipsnames,svgnames]{xcolor}%
}
\makeatother

\usepackage{colortbl}

\usepackage{graphicx}
\usepackage{wrapfig}

\definecolorset{RGB}{lyft}{}{Red,194,39,36;Sunset,202,53,33;Orange,205,68,20;Amber,200,117,42;Yellow,242,169,52;Citron,186,188,44;Lime,112,159,33;Green,56,139,31;Mint,45,118,56;Teal,52,133,135;Cyan,60,132,202;Blue,55,94,248;Indigo,64,13,247;Purple,115,42,248;Pink,176,25,145;Rose,176,32,75}

\usepackage{cite}

\usepackage{microtype}

\usepackage{array}
\usepackage{multirow}
\usepackage{booktabs}
\usepackage{makecell} 

\newcommand{\myparagraph}[1]{\textbf{\emph{#1}}.}

\ifcsname labelindent\endcsname

\fi
\usepackage[inline]{enumitem}

\usepackage{subfig}

\newenvironment{lenumerate}[2][]
{\begin{enumerate}[label=(#2\arabic*),leftmargin=0.2in,itemindent=0.15in,#1]}
{\end{enumerate}}

\setlist*[enumerate,1]{label={\itshape\arabic*)}}

\makeatletter
\newcommand{\paragraphswithstop}{%
\let\copyparagraph\paragraph%
\renewcommand\paragraph[1]{\copyparagraph{##1.}}%
}
\makeatother

\usepackage[framemethod=tikz]{mdframed}

\usepackage{suffix}

\usepackage{environ}

\makeatletter
\newsavebox{\boxifnotempty}
\newcommand{\displayifnotempty}[3]{\sbox\boxifnotempty{#2}\setbox0=\hbox{\usebox{\boxifnotempty}\unskip}%
\ifdim\wd0=0pt
\else
 #1\usebox{\boxifnotempty}#3%
\fi%
}
\newcommand{\ifempty}[2]{\setbox0=\hbox{#1\unskip}%
\ifdim\wd0=0pt%
 #2%
\fi%
}
\newcommand{\ifnotempty}[2]{\setbox0=\hbox{#1\unskip}%
\ifdim\wd0>0pt%
 #2%
\fi%
}
\makeatother

\usepackage{algpseudocode}
\usepackage{algorithm}

\usepackage{scrlfile}

\makeatletter
\newcommand*\newstoreddef[1]{
  \BeforeClosingMainAux{%
    \immediate\write\@auxout{%
      \string\restoredef{#1}{\csname #1\endcsname}%
    }%
  }%
}
\newcommand*{\restoredef}[2]{
  \expandafter\gdef\csname stored@#1\endcsname{#2}%
}
\newcommand*{\storeddef}[1]{
  \@ifundefined{stored@#1}{0}{\csname stored@#1\endcsname}%
}
\makeatother

\usepackage{pageslts}
\NewEnviron{tee}{\BODY\typeout{Marker Tee [start] ^^J \BODY ^^JMaker Tee [end]}}

\input{preamble/math}
\newcommand{\real}[1]{\mathbb{R}^{#1}{}}

\newcommand{\bmat}[1]{\begin{bmatrix}#1\end{bmatrix}}

\newcommand{\transpose}{^\mathrm{T}}

\newcommand{\inverse}{^{-1}}

\newcommand{\cross}[1]{[#1]_{\times}\!}

\DeclarePairedDelimiter{\norm}{\lVert}{\rVert}

\newcommand{\de}{\mathrm{d}}
\newcommand{\dert}[1][]{\frac{\de #1}{\de t}}

\newcommand{\vct}[1]{\mathbf{#1}}

\DeclareMathOperator{\diag}{diag}

\DeclareMathOperator*{\argmin}{\arg\!\min}

\DeclareMathOperator{\stack}{stack}

\DeclareMathOperator{\D}{D\!}

\newcommand{\intersect}{\cap}
\newcommand{\union}{\cup}

\providecommand{\vf}{\vct{f}}

\providecommand{\vn}{\vct{n}}

\providecommand{\vq}{\vct{q}}

\providecommand{\vr}{\vct{r}}

\providecommand{\vs}{\vct{s}}

\providecommand{\vu}{\vct{u}}

\providecommand{\vz}{\vct{z}}

\providecommand{\cE}{\mathcal{E}}
\providecommand{\cF}{\mathcal{F}}

\providecommand{\cI}{\mathcal{I}}

\providecommand{\cK}{\mathcal{K}}
\providecommand{\cL}{\mathcal{L}}

\providecommand{\cZ}{\mathcal{Z}}

\newcommand{\Fframe}[1]{{#1^{\cF}}}
\newcommand{\Eframe}[1]{{#1^{\cE}}}

\usepackage{units}

\newcommand{\newcolorlabel}[2]{%
  \expandafter\newcommand\csname #1\endcsname[1]{%
    \colorbox{#2}{\color{white}\textsf{\textbf{##1}}}}%
}

\newcommand{\newcommenter}[2]{%
  \expandafter\newcommand\csname #1\endcsname[1]{%
    \fcolorbox{#2}{#2}{\color{white}\textsf{\textbf{#1}}}
    {\color{#2}##1}}%
  \expandafter\newcommand\csname at#1\endcsname{%
    \fcolorbox{#2}{#2}{\color{white}\textsf{\textbf{@#1}}}
    {\color{#2}}}%
  \expandafter\newcommand\csname #1hl\endcsname[2]{%
    \colorbox{#2}{\color{white}\textsf{\textbf{#1}}}\sethlcolor{Azure2}\hl{##2}~%
    \expandafter\ifx\csname commentarrow\endcsname\relax$\leftarrow$\else \commentarrow[#2]\fi~%
    {\color{#2}##1}}%
  \expandafter\newcommand\csname #1st\endcsname[2]{%
    \colorbox{#2}{\color{white}\textsf{\textbf{#1}}}\sout{##2}~%
    \expandafter\ifx\csname commentarrow\endcsname\relax$\leftarrow$\else \commentarrow[#2]\fi~%
    {\color{#2}##1}}%
}
\newcommenter{TODO}{DodgerBlue1}
\newcommenter{rtron}{Green3}
\newcommenter{zyang}{orange}

\usepackage{comment}

\usepackage{pdfcomment}

\usepackage{soul}

\usepackage[normalem]{ulem}

\usepackage{csquotes}

\usepackage{tikz}
\usetikzlibrary{calc}
\usetikzlibrary{matrix}
\usetikzlibrary{chains}
\usetikzlibrary{shapes.geometric}
\usetikzlibrary{arrows.meta}
\usetikzlibrary{decorations.pathreplacing}
\usetikzlibrary{backgrounds}

\tikzset{
  dim above/.style={to path={\pgfextra{
        \pgfinterruptpath
        \draw[>=latex,|->|] let
        \p1=($(\tikztostart)!1.5em!90:(\tikztotarget)$),
        \p2=($(\tikztotarget)!1.5em!-90:(\tikztostart)$)
        in(\p1) -- (\p2) node[pos=.5,sloped,above]{#1};
        \endpgfinterruptpath
      }
    }
  },
  dim double above/.style={to path={\pgfextra{
        \pgfinterruptpath
        \draw[>=latex,|->|] let
        \p1=($(\tikztostart)!3em!90:(\tikztotarget)$),
        \p2=($(\tikztotarget)!3em!-90:(\tikztostart)$)
        in(\p1) -- (\p2) node[pos=.5,sloped,above]{#1};
        \endpgfinterruptpath
      }
    }
  },
  dim below/.style={to path={\pgfextra{
        \pgfinterruptpath
        \draw[>=latex,|->|] let 
        \p1=($(\tikztostart)!-1em!-90:(\tikztotarget)$),
        \p2=($(\tikztotarget)!-1em!90:(\tikztostart)$)
        in (\p1) -- (\p2) node[pos=.5,sloped,below]{#1};
        \endpgfinterruptpath
      }
    }
  },
}

\tikzset{
    right angle quadrant/.code={
        \pgfmathsetmacro\quadranta{{1,1,-1,-1}[#1-1]}     
        \pgfmathsetmacro\quadrantb{{1,-1,-1,1}[#1-1]}},
    right angle quadrant=1, 
    right angle length/.code={\def\rightanglelength{#1}},   
    right angle length=2ex, 
    right angle symbol/.style n args={3}{
        insert path={
            let \p0 = ($(#1)!(#3)!(#2)$) in     
                let \p1 = ($(\p0)!\quadranta*\rightanglelength!(#3)$), 
                \p2 = ($(\p0)!\quadrantb*\rightanglelength!(#2)$) in 
                let \p3 = ($(\p1)+(\p2)-(\p0)$) in  
            (\p1) -- (\p3) -- (\p2)
        }
    }
}

\newcommand{\pgfextractangle}[3]{%
    \pgfmathanglebetweenpoints{\pgfpointanchor{#2}{center}}
                              {\pgfpointanchor{#3}{center}}
    \global\let#1\pgfmathresult  
}

\usetikzlibrary{shapes.arrows}
\newcommand{\commentarrow}[1][Azure4]{\tikz[baseline=-3pt]{\node[shape border uses incircle, fill=#1,rotate=180,single arrow, inner sep=1pt, minimum size=6pt, single arrow head extend=2pt]{};}}

\tikzset{ax/.style={-latex,line width=2pt}}

\tikzset{camera/.style={fill=Sienna1,fill opacity=0.5},%
image plane/.style={draw=RoyalBlue3,line width=2pt}}

\newcommand{\rrtstar}{$\texttt{RRT}^\texttt{*}$}

\newtheorem{constraint}{ADMM constraint}
\newtheorem{example}{Example}

\usepackage[export]{adjustbox} 
\usepackage[capitalise]{cleveref}
\usepackage{xcolor}

\graphicspath{{Figs/}}

\DeclareMathAlphabet{\mathcal}{OMS}{cmsy}{m}{n}
\def\sZ{\mathcal{Z}}
\title{\LARGE \bf

Enhancing Security in Multi-Robot Systems through Co-Observation Planning, Reachability Analysis, and Network Flow}

\author{Ziqi Yang, Roberto Tron \IEEEmembership{Member, IEEE} 
\thanks{This project is supported by the National Science Foundation grant "CPS: Medium: Collaborative Research: Multiagent Physical Cognition and Control Synthesis Against Cyber Attacks" (Award number 1932162).}
\thanks{Ziqi Yang is with the Department of Systems Engineering,
Boston University, Boston, MA 02215 USA (e-mail: zy259@bu.edu).
}
\thanks{Roberto Tron is with the Faculty of Mechanical Engineering and Systems Engineering, Boston University, Boston, MA 02215 USA (e-mail:
tron@bu.edu).}}

\begin{document}

\maketitle
\thispagestyle{empty}
\pagestyle{empty}


\begin{abstract}
We address security challenges in multi-robot systems (MRS) where adversaries may take control of compromised robots to gain unauthorized physical access to forbidden regions. We propose a novel multi-robot optimal planning algorithm that integrates mutual observations and introduces \emph{reachability constraints} and \emph{co-observations} between robots to enhance security. Our formulation can guarantee that, even with adversarial movements, compromised robots cannot breach forbidden regions without missing scheduled co-observations. We use ellipsoidal over-approximations to reachability regions for efficient intersection checking and gradient computation. Furthermore, to enhance system resilience and tackle feasibility challenges, we introduce the notion of \emph{sub-teams}; these cohesive units replace individual robots along each trajectory; we plan \emph{cross-trajectory co-observation} that use redundant robots to switch between different trajectories, securing multiple sub-teams without requiring modifications to the plans. We formulate cross-trajectories plans by solving a \emph{network-flow vertex path coverage problem} on the \emph{checkpoint graph} generated from the original unsecured MRS trajectories, providing the same security guarantees against plan-deviation attacks. We demonstrate the effectiveness and robustness of our proposed algorithm through simulations, which significantly strengthen the security of multi-robot systems in the face of adversarial threats.
\end{abstract}

\section{Introduction}\label{sec:introduction}
Multi-robot systems (MRS) have found wide applications in various fields. Although offering numerous advantages, the distributed nature and dependence on network communication make the MRS vulnerable to cyber threats, such as unauthorized access, malicious attacks, and data manipulation \cite{brunner2010infiltrating}. This paper addresses a specific scenario in which robots are compromised by physically masquerading attackers and directed toward \emph{forbidden regions}, which may contain security-sensitive equipment or human workers. A countermeasure to these \emph{plan-deviation attacks} \cite{wardega2019resilience, wardega2023byzantine, wardega2023hola} is to use the onboard sensing capabilities of the robots themselves to perform inter-robot \emph{co-observations} and detect unusual behavior. These mutual observations establish a \emph{co-observation schedule} alongside the planned path, ensuring that any attempts by a compromised robot to violate safety constraints (such as transgressing into forbidden regions) would break the observation plan and be promptly detected.

A preliminary version of this approach is presented in \cite{yang2020multi,yang2021multi}. That work extended the grid-world solution from \cite{wardega2019resilience} to continuous configuration space, and incorporated the co-observation planning as constraints in the alternating direction method of multipliers (ADMM)-based trajectory optimization solver to accommodate for more flexible objectives. However, such plans are not guaranteed to exist, and the secured trajectory always comes at the cost of overall system performance. As an extension of the previous works \cite{yang2020multi,yang2021multi}, we introduce reachability analysis and network flow to the security problem to find a balance between security and performance.

In this paper, we incorporate additional reachability analysis during the planning phase, incorporating constraints based on sets of locations that agents could potentially reach, referred to as \emph{reachability regions}, as a novel perspective to the existing literature. This approach enforces an empty intersection between forbidden regions and the reachability region during trajectory optimization, preventing undetected attacks if an adversary gains control of the robots. We utilize an ellipsoidal boundary to constrain the search space and formulate the ellipsoid as the reachability region constraint. We present a mathematical formulation of the reachability regions as spatio-temporal constraints compatible with the solver in \cite{yang2020multi}. 

Furthermore, we address the feasibility and optimality challenges. For a MRS with unsecured optimal trajectories (without security constraints), we introduce redundant robots and form them into \emph{sub-teams} with the original ones. These redundant robots are assigned to establish additional co-observations, termed \emph{cross-trajectory co-observations}, within and across different sub-teams. The proposed algorithm focuses on the movement plans for the additional robots, distinct from those dedicated to task objectives, indicating when they should stay with their current sub-team and when they should deviate to join another. This strategy allows sub-teams to preserve the optimal unsecured trajectories (as illustrated in \cref{fig:cross-traj-comparison-set}) without requiring the entire \emph{sub-team} to maintain close proximity to other teams during co-observation events. The cross-trajectory co-observation problem is transformed as a Multi-Agent Pathfinding (MAPF) problem on roadmap (represented as directed graph) and solved as a network flow problem \cite{yu2013multi}.

\begin{figure}
	\centering
    \subfloat[Plan-deviation attack\label{fig:example-plan-dev}]{\includegraphics[width = 0.32\linewidth]{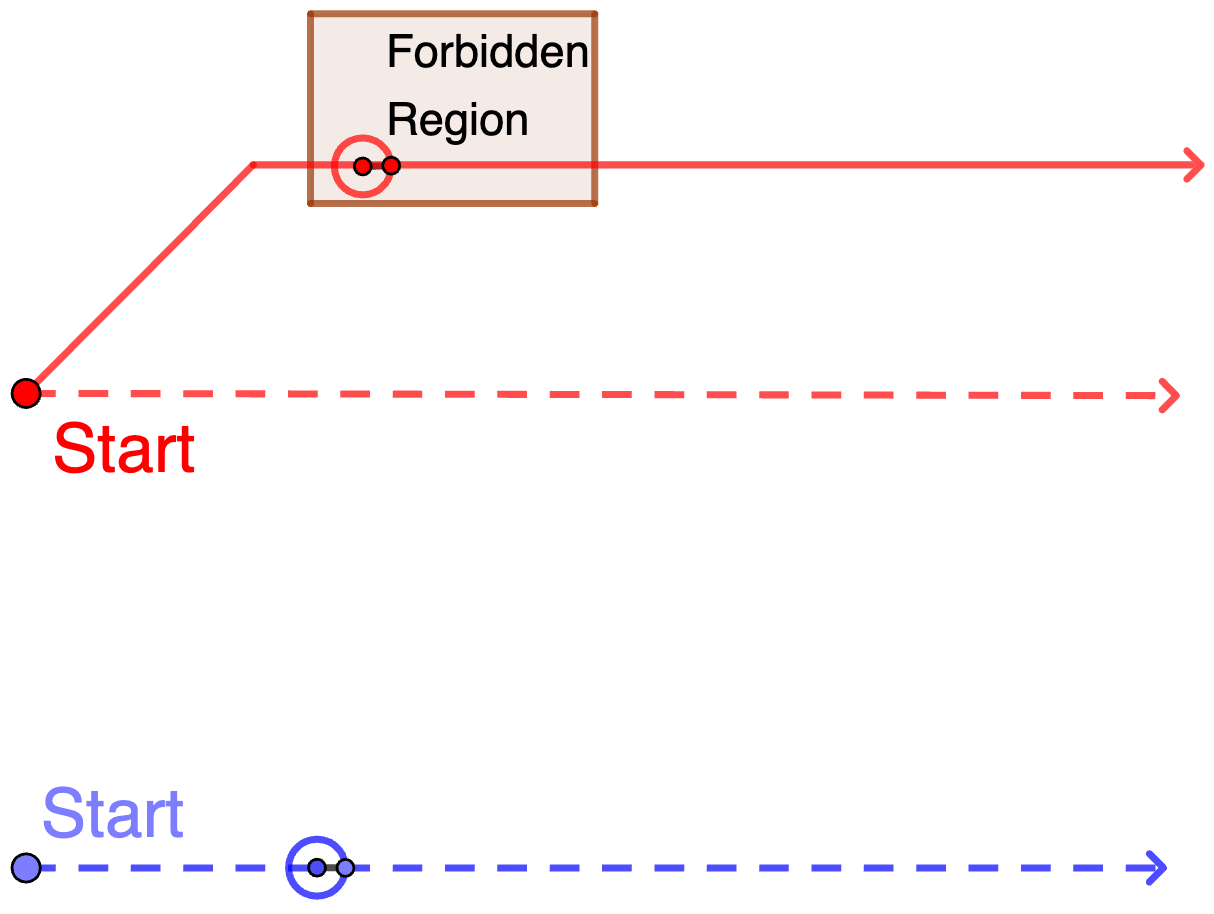}}
    \subfloat[Co-observation \label{fig:example-co-observation}]{\includegraphics[width=0.32\linewidth]{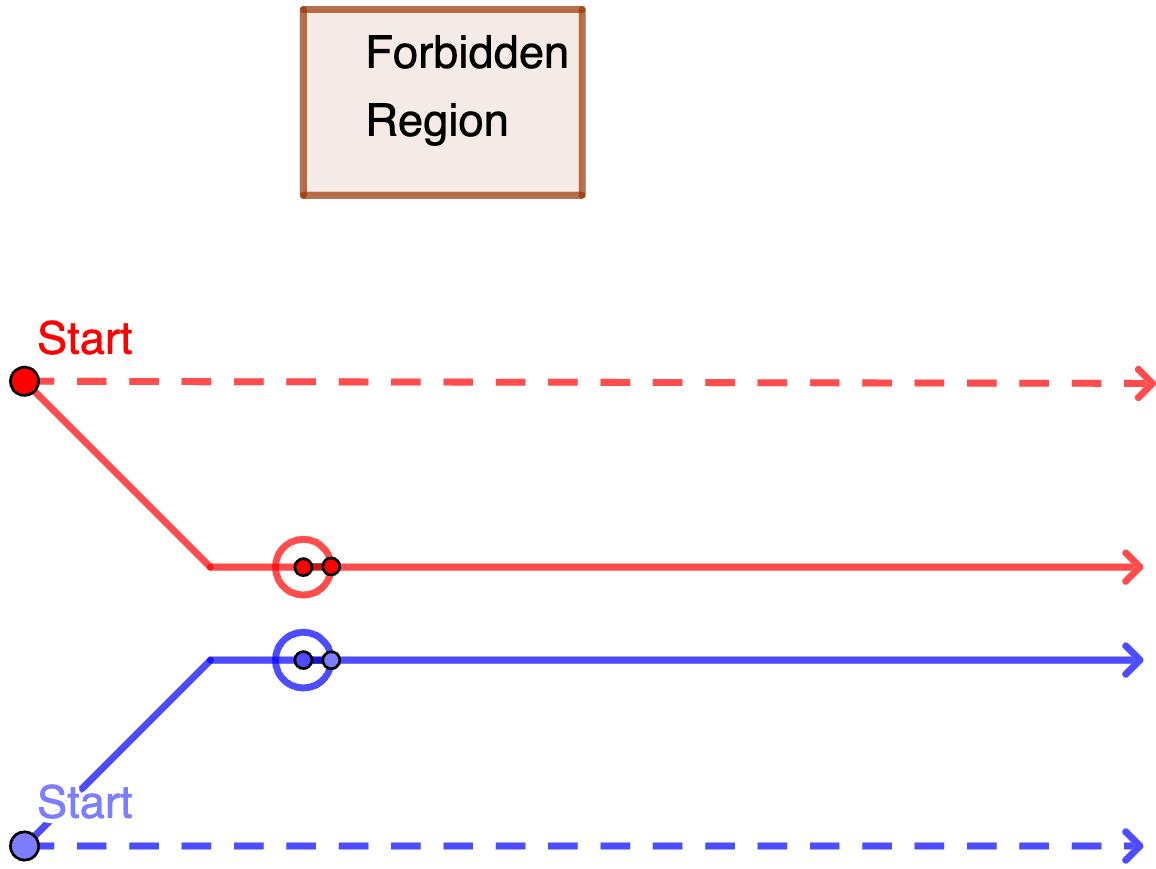}}
    \subfloat[Cross-trajectory co-observation \label{fig:example-cross-traj}]{\includegraphics[width=0.32\linewidth]{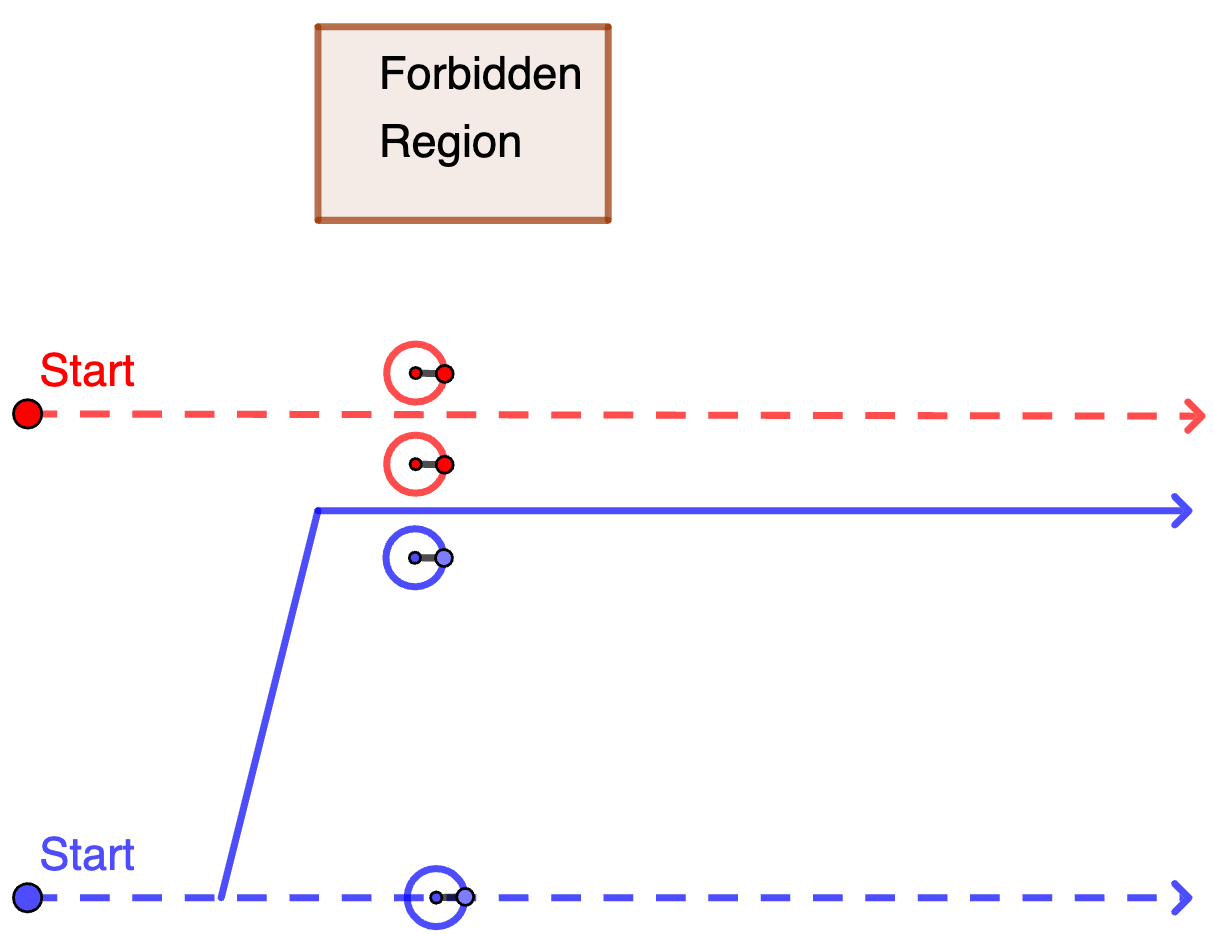}}
    
    \caption{ (\ref{fig:example-plan-dev}) The attacker deviates the red robot into a forbidden region.(\ref{fig:example-co-observation}) Red and blue robots are scheduled to co-observe each other during the task, resulting in a new secured trajectory (solid lines) replacing the optimal ones (dashed line). (\ref{fig:example-cross-traj}) The blue team sends one robot to observe the red team (solid blue line) while having the rest of the robots follow the optimal trajectory.}\label{fig:cross-traj-comparison-set}
\end{figure}

\noindent\myparagraph{Related research} The trajectory planning problem of MRS remains as a subject of intense study for many decades, and optimization is a common approach in such areas. Optimization-based approaches are customizable to a variety of constraints (e.g. speed limit, avoiding obstacles) and task specifications (e.g. maximum surveillance coverage, minimal energy cost). In contrast to our work in \cref{sec:ADMM-planning}, many motion planning tasks require a non-convex constraint problem formulation, most contributors focused on convex problems, and only allow for a few types of pre-specified non-convex constraints through convexification\cite{liu2014solving}\cite{VanParys2016}\cite{Schulman2014}. Several optimization techniques like mixed integer programs with quadratic terms (MIQP)~\cite{mellinger2012mixed} and ADMM~\cite{bento2013message} have been used to reduce computational complexity and to incorporate more complex non-convex constraints.

Graph-based search is another extensively explored approach in trajectory planning problems, such as formation control \cite{tanner2004leader,hu2019distributed} and MAPF problems \cite{stern2019multi}. In traditional MAPF formulations, environments are abstracted as graphs, with nodes representing positions and edges denoting possible transitions between positions, and solved as a network flow problem \cite{yu2013multi,yu2016optimal}. This formulation allows for the application of combinatorial network flow algorithms and linear program techniques, offering efficient and more flexible solutions to the planning problem.

Reachability analysis is essential for security and safety verification in cyber-physical systems (CPSs) \cite{gueguen2009safety, ding2020secure}, often involving an over-approximation of reachable space to verify safety properties. Geometric methods like zonotopes and ellipsoids are commonly used to enclose reachability sets compactly \cite{kurzhanski2000ellipsoidal, lakhal2019interval, maiga2015comprehensive}. For online safety assessments against cyber attacks, \cite{kwon2017reachability} compute reachable CPS states under attacks and compared them with a safe region based on state estimation. By incorporating an additional security measure that provides two secured states, the reachability region in our work is over-approximated using ellipsoids. This is inspired by the ellipsoidal \emph{heuristic sampling domain} in \cite{gammell2014informed} for the $\mathtt{RRT^*}$ algorithm to simplify the sampling region between start and goal locations. Ellipsoids are also used in other path-planning methods like Iterative Regional Inflation by Semidefinite programming (IRIS) \cite{deits2015computing, ray2022free} and the Safe Flight Corridor (SFC) \cite{liu2017planning, fan2024flying}. Differing from the reachability region that requires mapping all reachable states given several known states, the ellipsoids of IRIS and SFC focus on approximating the safe collision-free space rather than addressing security applications.

\noindent\myparagraph{Paper contributions} 
Two main contributions have been presented. 
\begin{itemize}
  \item We present an innovative method to integrate reachability analysis into the ADMM-based optimal trajectory solver for multi-robot systems, preventing attackers from executing undetected attacks by simultaneously entering forbidden regions and adhering to co-observation schedules.
  \item We introduce additional robots to form \emph{sub-teams} for both intra-sub-team and cross-sub-team co-observations. A new co-observation planning algorithm is formulated that can generate a resilient multi-robot trajectory with a co-observation plan that still preserves the optimal performance against arbitrary tasks. We also find the minimum redundant robots required for the security.
\end{itemize}

The rest of this paper is organized as follows. \cref{sec:ADMM-planning} introduces the ADMM-based optimal trajectory solver and the security constraints. \cref{sec:cross-trajectory} introduces the enhanced security planning algorithm through cross-trajectory co-observation. \cref{sec:summary} concludes this article.

\noindent\myparagraph{Notation} 
In this paper, we use non-bold symbols to denote single-agent states (e.g. $q_{ij}$) and scalars, and bold symbols to represent aggregated states of single robot and multiple robots (e.g. $\vq$).

\section{Secured Multi-robot trajectory planning}\label{sec:ADMM-planning}
We formulate the planning problem as an optimal trajectory optimization problem to minimize arbitrary smooth objective functions. We denote the trajectory as $\vq_i = [q_{i0}\dots q_{iT}]\in\real{m\times T}$, where $q_{ij}\in\real{m}$ is the waypoint of agent $i$ in a $m$ dimensional state space, $T$ is the time horizon. For a total of $n_p$ robots, trajectories can be represented as an aggregated vector $\vq = \stack(\vq_1,\dots,\vq_{n_p})\in \real{m n_p\times T}$, where $\stack(\cdot)$ denotes the vertical stacking operation. 
The overall goal is to minimize or maximize an objective function $\varPhi(\vq):\real{n m T} \to \real{}$ under a set of nonlinear constraints described by a set $\Omega$, which is given by the intersection of spatio-temporal sets given by traditional path planning constraints and the security constraints (co-observation schedule, reachability analysis). Formally:
\begin{equation}\label{eq:general-problem}
	\begin{split}
		\min/\max & \quad \varPhi(\vq)\\
		\textrm{subject to} &\quad \vq \in \Omega.
	\end{split}
\end{equation}
To give a concrete example of the cost $\varPhi$ and the set $\Omega$, we introduce a representative application that will be used for all the simulations throughout the paper.

\begin{example}\label{example:map_exploration}
Robots are tasked with navigating an unknown task space to collect sensory data and reconstruct a field (as illustrated in~\cref{fig:illustration}). The task space is modeled as a grid, where each grid point has an associated value tracked by a Kalman Filter (KF) \cite{anderson2012optimal}. The KF estimates uncertainty at each point through its covariance $P_j$ updated based on the measurements taken by robots along the trajectory $\vq$. Measurement quality, modeled by a Gaussian radial basis function, decreases with distance from the robot. The optimization objective $\varPhi(\vq)=\max_j P_j(\vq)$ is to minimize the maximum uncertainty in reconstructing the field (detailed in \cite{yang2020multi}). 
\end{example}
\begin{figure}
  \centering
  \subfloat[Map exploration task\label{fig:illustration}]{\includegraphics[width = 0.45\linewidth,, trim = 0.3cm 1.3cm 0.3cm 2cm,valign=c]{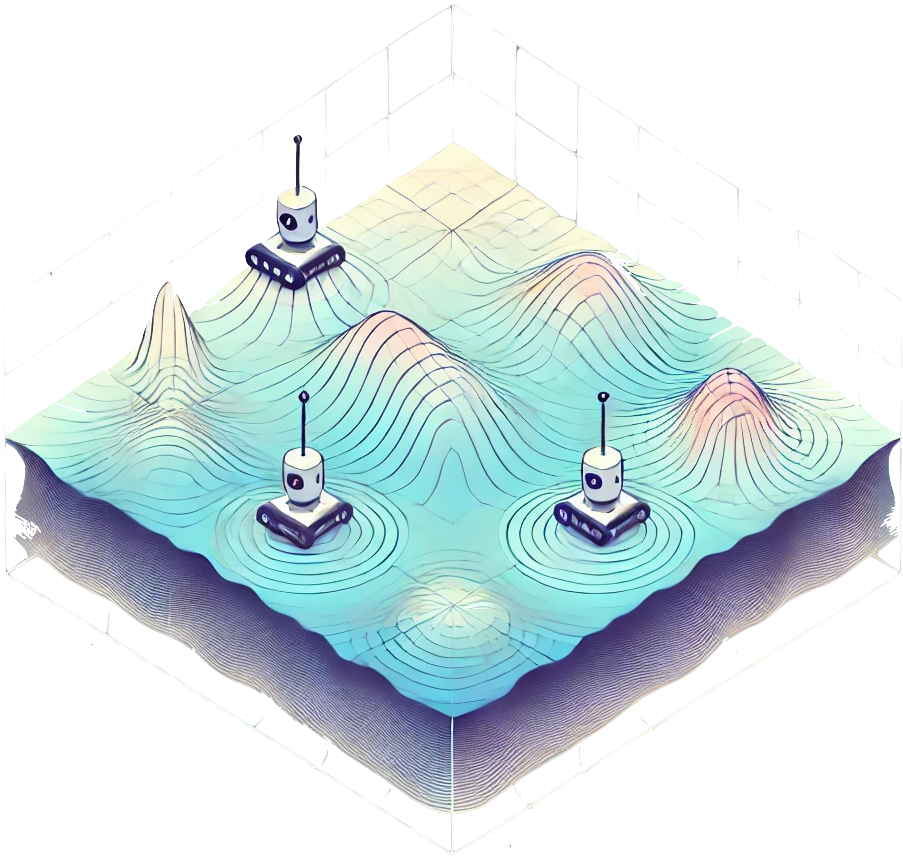}} 
  \subfloat[Continuous world trajectory \label{fig:SecurityBreak}]{\includegraphics[width=0.45\linewidth, trim = 0.3cm 1.3cm 0.3cm 2cm, valign=c]{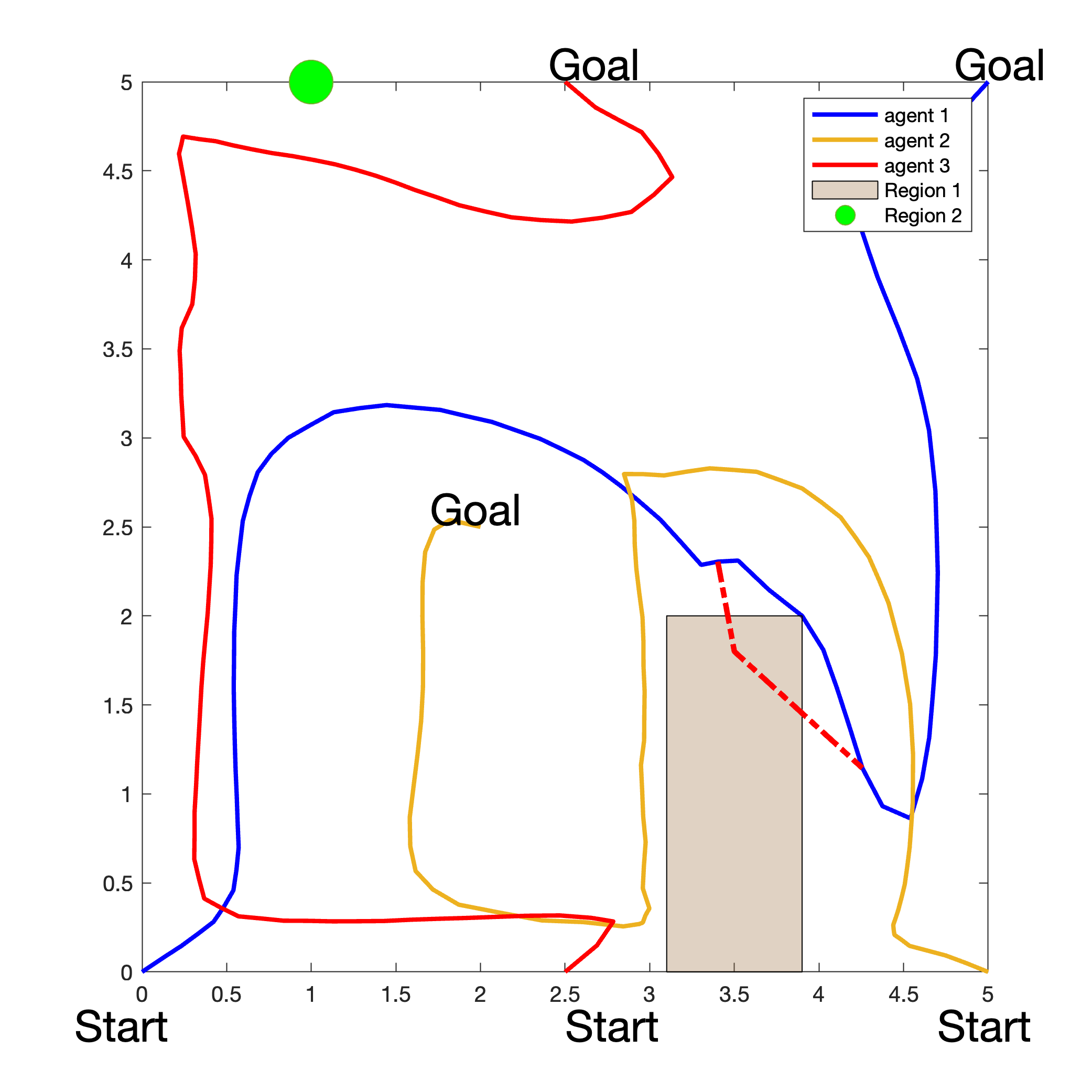}}
  \caption{(\ref{fig:illustration}) Illustration of a 3 robot map exploration task case. (\ref{fig:SecurityBreak}) The unsecured trajectory design optimized for a map exploration task. Potential security breaches, indicated by red dashed lines, highlight paths that could allow unauthorized access to forbidden regions.}
\end{figure}

 Ideally, robots are expected to spread out in a Boustrophedon pattern to efficiently explore the task space, optimize coverage, and minimize reconstruction uncertainty (\cref{fig:SecurityBreak}). However, we consider the possibility of threats following the \emph{physical masquerade attack} model \cite{wardega2019resilience}, where a compromised insider (robot) masquerading as a properly functioning agent, attempts to gain access to unauthorized locations without detection. The attacker leverages full knowledge of the motion plan and exploits the compromised robot to provide false self-reports to the central entity (CE). These malicious deviations remain undetected as long as observations from uncompromised robots align with the motion plan. The corresponding security requirement can then be formally defined as:
\begin{definition}\label{def:secured-plan}
  A multi-robot trajectory plan is secured against plan-deviation attacks if it ensures that any potential deviations to these forbidden regions will cause the corresponding robot to miss their next co-observation with other robots. 
\end{definition}

\subsection{Preliminaries}
\subsubsection{Differentials}
We define the differential of a map $f(x):\real{d_1}\to\real{d_2}$ at a point $x_0$ as the unique matrix $\partial_x f \in\real{d_1\times d_2}$ such that
\begin{equation}\label{equ:dt_to_dx}
  \left.\dert f\bigl(x(t)\bigr)\right|_{t=0}=\partial_x f\bigl(x(0)\bigr) \dot{x}(0),
\end{equation}
where $t\mapsto x(t)\in\real{d_1}$ is a smooth parametric curve such that $x(0)=x$ with any arbitrary tangent $\dot{x}(0)$. For brevity, we will use $\dot f$ for $\dert f$ and $\partial_x f$ for $\frac{\partial f}{\partial x}$. 

With a slight abuse of notation, we use the same notation $\partial_xf$ for the differential of a matrix-valued function with scalar arguments $f:\real{d_1 \times d_3}\to\real{d_2}$.  Note that in this case \eqref{equ:dt_to_dx} is still formally correct, although the RHS needs to be interpreted as applying $\partial_x f$ as a linear operator to $\dot{x}$.

\subsubsection{Alternating Directions Method of Multipliers (ADMM)}\label{chapter:ADMM review}
The basic idea of the ADMM-based solver introduced in \cite{yang2020multi} is to separate the constraints from the objective function using a different set of variables $\vz$, then solve separately in \eqref{eq:general-problem}.
More specifically, we rewrite the constraint $\vq\in\varOmega$ using an indicator function $\varTheta$ and include it in the objective function (details can be found in \cite{yang2020multi}). 
We allow $\vz = D(\vq)$ to replicate an arbitrary function of the main variables $\vq$ (instead of being an exact copy in general ADMM formulation), to transform constraint $\vq\in\Omega$ to $D(\vq) \in \sZ$ to allow for an easier projection step in \cref{eq:z-update}. In summary, we transform \cref{eq:general-problem} into
\begin{equation}\label{eq:ADMMSetConstraint_modified}
	\begin{aligned}
		&\max\quad \varPhi(\vq)+\varTheta(\vz), \\
		& \begin{array}{r@{\quad}c}
			s.t.& D(\vq)-\vz=0,
		\end{array} 
	\end{aligned}
\end{equation}
where $D(\vq)= [D_1(\vq)^T,\dots,D_l(\vq)^T]^T$ is a vertical concatenation of different functions for different constraints. This makes each constraint set $\mathcal{Z}_i$ independent and thus can be computed separately in later updating steps. $D_i(\vq)$ is chosen such that the new constraint set~$\cZ_i$ becomes simple to compute which is illustrated in later sections.

The update steps of the algorithm are then derived as \cite{Boyd2011}:
\begin{subequations}\label{eq:ADMMupdate}
	\begin{align}
		\vq^{k+1}&:=\argmin_\vq(\varPhi(\vq^k) +\frac{\rho}{2}\norm{D(\vq)-\vz^k+\vu^k}_2^2),\label{eq:admm-mod-update-q}\\
		\vz^{k+1}&:=\Pi_\mathcal{Z}(D(\vq^{k+1})+\vu^k), \label{eq:z-update} \\
		\vu^{k+1}&:=\vu^k+D(\vq^{k+1})-\vz^{k+1}, \label{eq:u-update}
	\end{align}
\end{subequations}
where $\Pi_\mathcal{Z}$ is the new projection function to the modified constraint set $\mathcal{Z}$, $\vu$ represents a scaled dual variable that, intuitively, accumulates the sum of primal residuals
\begin{equation}\label{eq:primal-residual}
	\vr^{k}=D(\vq^{k+1})-\vz^{k+1}.
\end{equation}
Checking the primal residuals alongside the dual residuals 
\begin{equation}\label{eq:dual-residual}
	\vs^{k}=-\rho(\vz^{k}-\vz^{k-1})
\end{equation}
after each iteration, the steps are reiterated until convergence when the primal and dual residuals are small, or divergence when primal and dual residuals remain large after a fixed number of iterations.
\begin{remark}
In practical application, \eqref{eq:admm-mod-update-q} is solved iteratively via a nonlinear optimization solver such as \emph{fmincon}~\cite{MATLAB:fmincon}. The use of explicit gradients of the constraint functions (derived below) greatly enhances the solver's efficiency and accuracy (with respect to the use of numerical differentiation). At the same time, \eqref{eq:z-update} and~\eqref{eq:u-update} have closed form solutions.
\end{remark}

We now provide the functions $D(\vq)$, the sets $\cZ$, and the corresponding projection operators $\Pi_\cZ$ for security constraints including co-observation security constraints (\cref{sec:co-observation-constraint}), and reachability constraints (\cref{sec:ellipsoid-point}-\cref{eq:region_ellipsoid_constraint}). The latter is based on the definition of \emph{ellipse-region-constraint} (\cref{sec:reachability}).
Formulation of traditional path planning constraints like velocity constraints and convex obstacle constraints can be found in \cite{yang2020multi}, thus are omitted in this paper.

\subsection{Co-observation schedule constraint}\label{sec:co-observation-constraint}
The co-observation constraint ensures that two robots come into close proximity at scheduled times to observe each other's behavior. This constraint is represented as a relative distance requirement between the two robots at a specific time instant, ensuring they are within a defined radius to inspect each other or exchange data.
\begin{constraint}[Co-observation constraint]\label{constraint:coobservation}
\begin{align}
D(\vq) &= \overrightarrow{q_{aj}q_{bj}}, \label{eq:coobservation_constraint}\\
  \sZ &= \{\vz \mid \norm{\vz} \leq d_{max} \},\\
   \Pi_\sZ(z) & = \begin{cases}
d_{max}\frac{\vz}{\norm{\vz}} &\text{if } \norm{\vz} > d_{max},\\
\vz	& otherwise,
\end{cases}
\end{align}
where $a,b$ are the indices of the pair of agents required for a mutual inspection, $d_{max}$ is the maximum distance between a pair of robots that the security can be ensured through co-observation.
\end{constraint}
The locations $q_{aj}$ and $q_{bj}$ where the co-observation is performed are computed as part of the optimization.

\subsection{Definition of ellipsoidal reachability regions}\label{sec:reachability}

In this section, we define \emph{ellipsoidal reachability regions} with respect to a pair of locations along a trajectory. We then use the ellipsoids to introduce various reachability constraints with respect to other geometric entities, such as points, lines, line segments, and polygons. These constraints, combined with the co-observation constraint~\ref{constraint:coobservation}, form the basis of the secured ADMM planning is presented in \cref{chapter:ADMM review}.

\begin{definition}\label{sec:ellipsoidal definition}
Consider a robot $i$ starting from $q_{1}$ at time $t_1$ and reaching $q_{2}$ at time $t_2$. The \emph{reachability region} between $t_1$ and $t_2$ is defined as the set of points $q'$ in the free configuration space for which there exists a kinematically feasible trajectory containing $q'$.
\end{definition}

For simplicity, in this paper we consider only a maximum velocity constraint $v_{max}$; the reachability region for $q'$ can then be over-approximated by the following (the over-approximation is obtained by neglecting physical obstacles):·
\begin{definition}\label{def:Reachability}
	The \emph{reachability ellipsoid $\cE$} is defined as the region $\mathcal{E}(q_1,q_2,t_{1},t_{2})=\{\tilde{q}\in\mathbb{R}^n: d(q_1,\tilde{q})+d(\tilde{q},q_2)<2a\}$, where $a=\frac{v_{max}}{2}(t_2-t_1)$, and $d(\cdot,\cdot)$ denotes the Euclidean distance between two points.
\end{definition}
This region is an ellipsoid because it represents the set of points $q'$ whose sum of distances two \emph{foci} $q_1$ and $q_2$ is less than $2a$, where $a$ is the major radius of the ellipsoid. 
 \begin{figure}
    \centering
    \subfloat[Showcase of a reachability region\label{fig:EllipseConstraintExample}]{\includegraphics[width=0.47\linewidth]{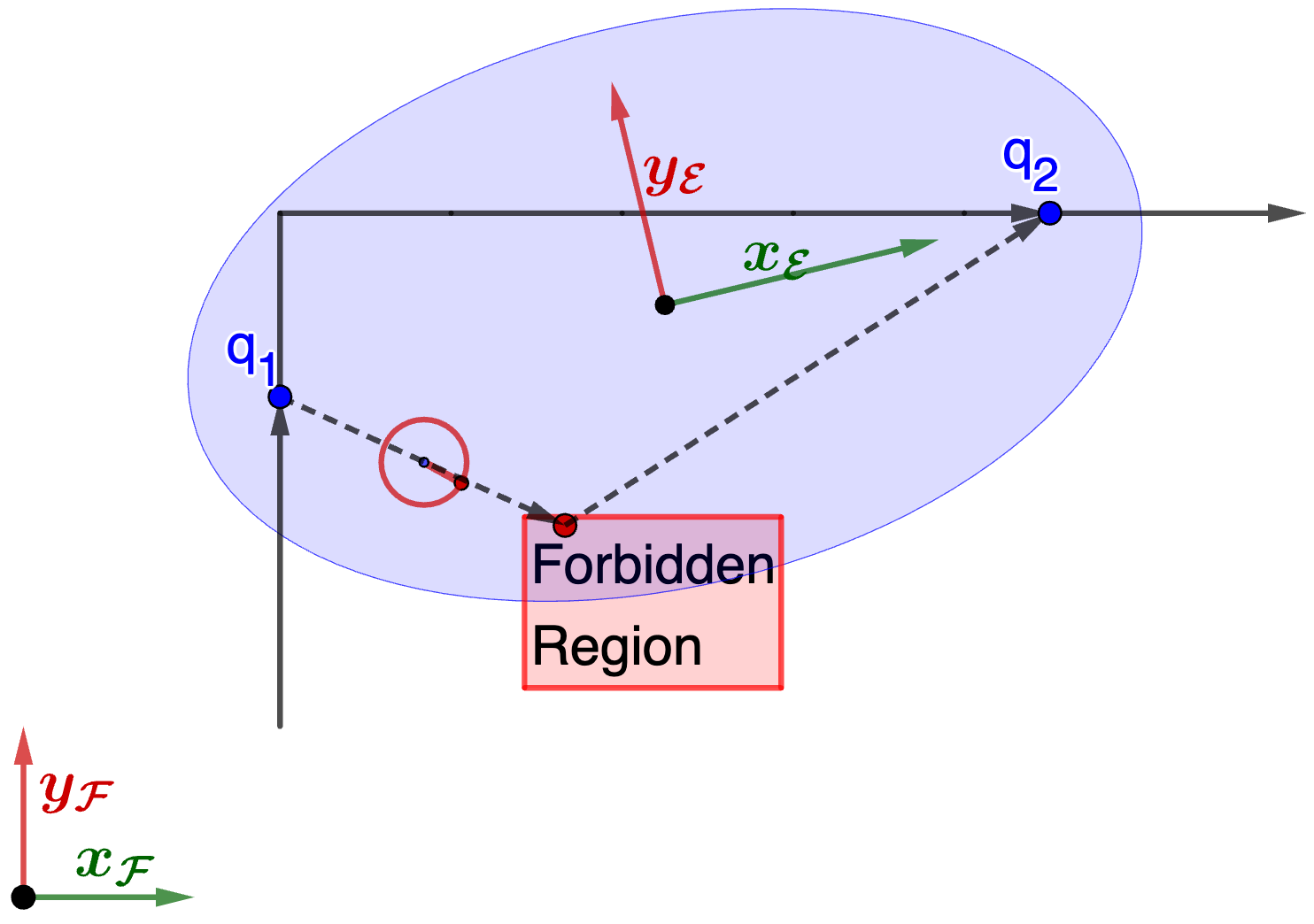}}
    \subfloat[Point-ellipsoid constraint \label{fig:Ellipse-to-point}]{\includegraphics[width=0.47\linewidth]{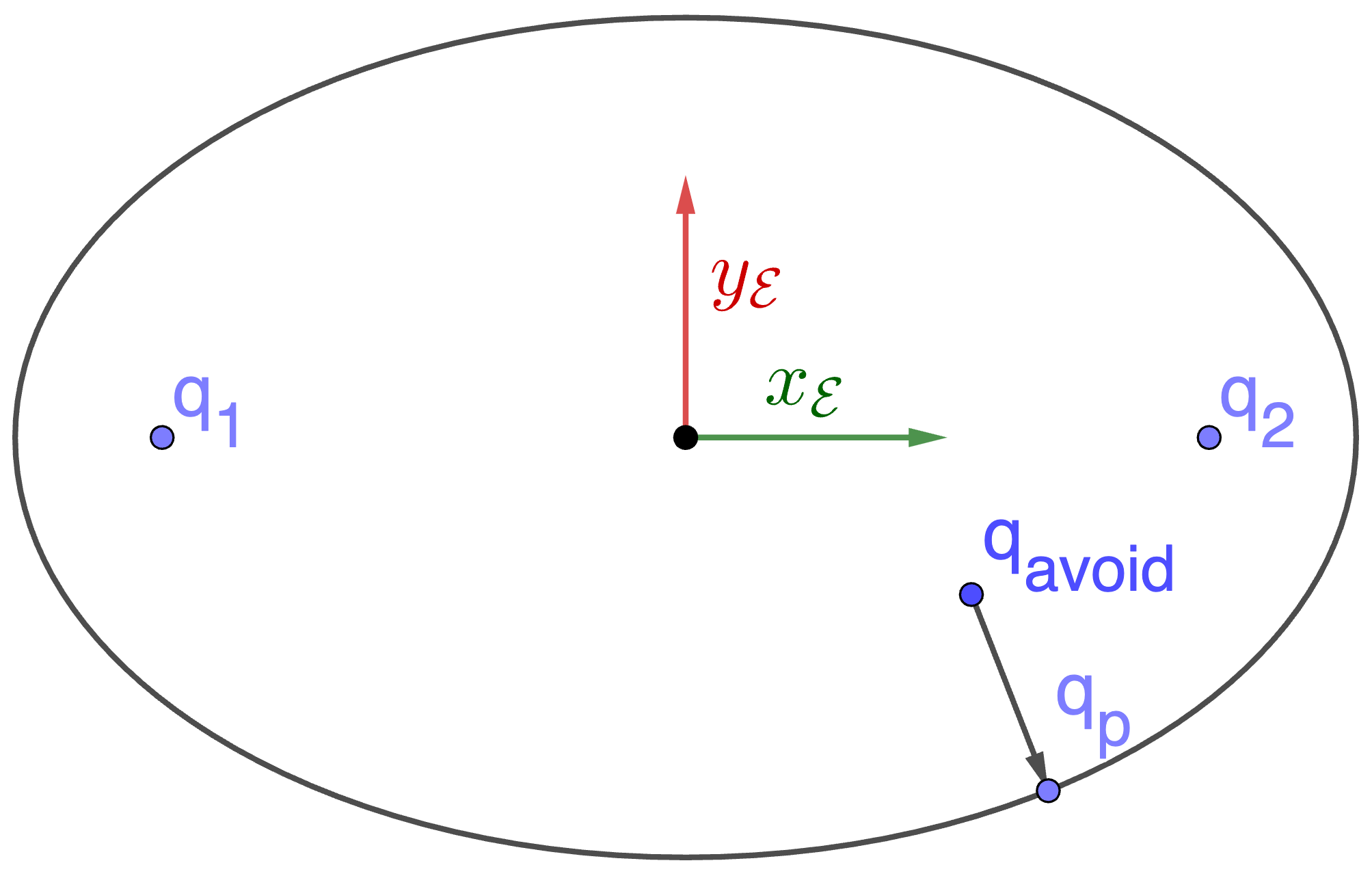}}

    \subfloat[Plane-ellipse constraint \label{fig:Ellipse-to-plane}]{ \includegraphics[width=0.47\linewidth]{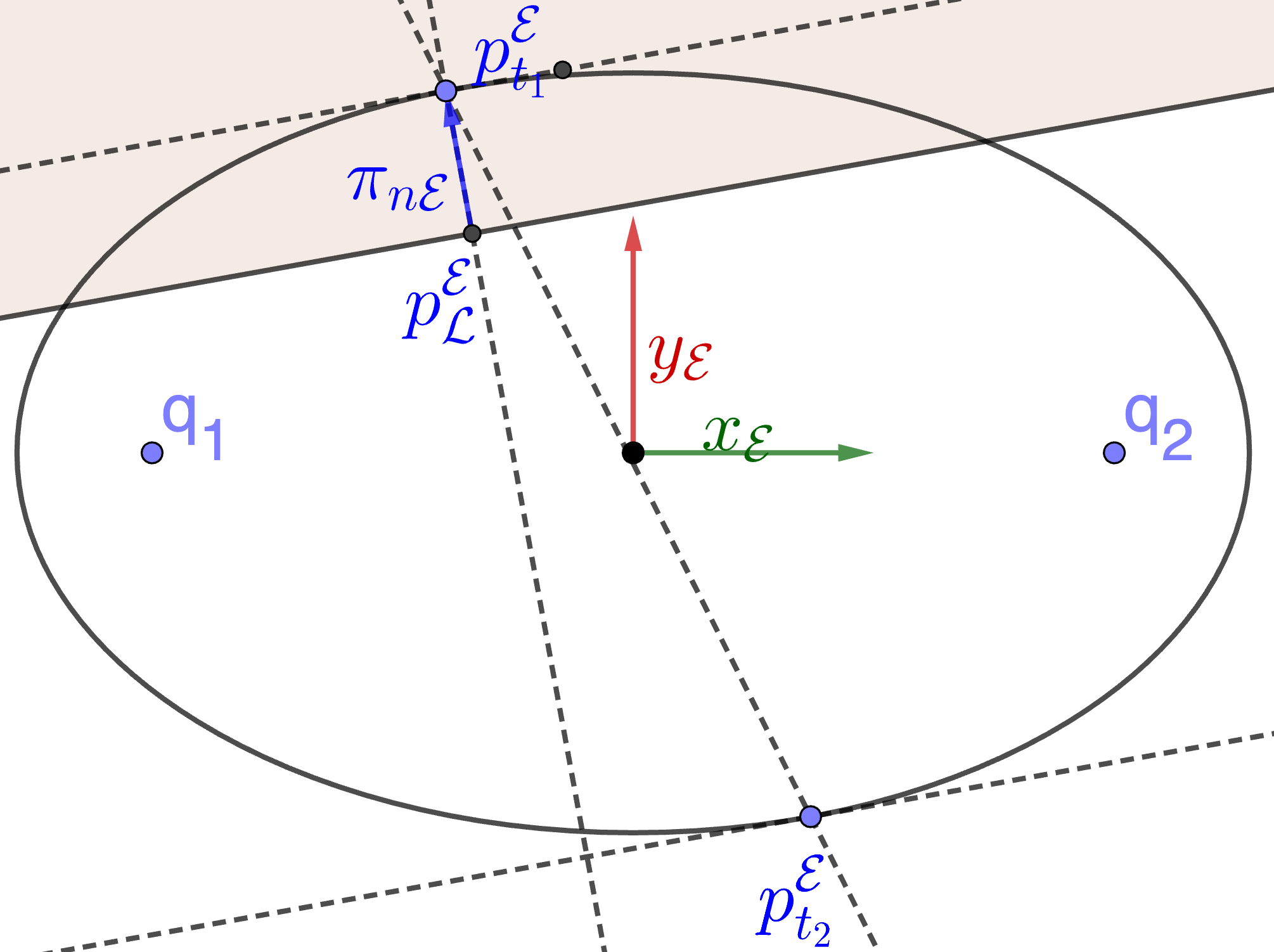}}
    \subfloat[Polygon-ellipse constraint \label{fig:Ellipse-to-safezone}]{\includegraphics[width=0.47\linewidth]{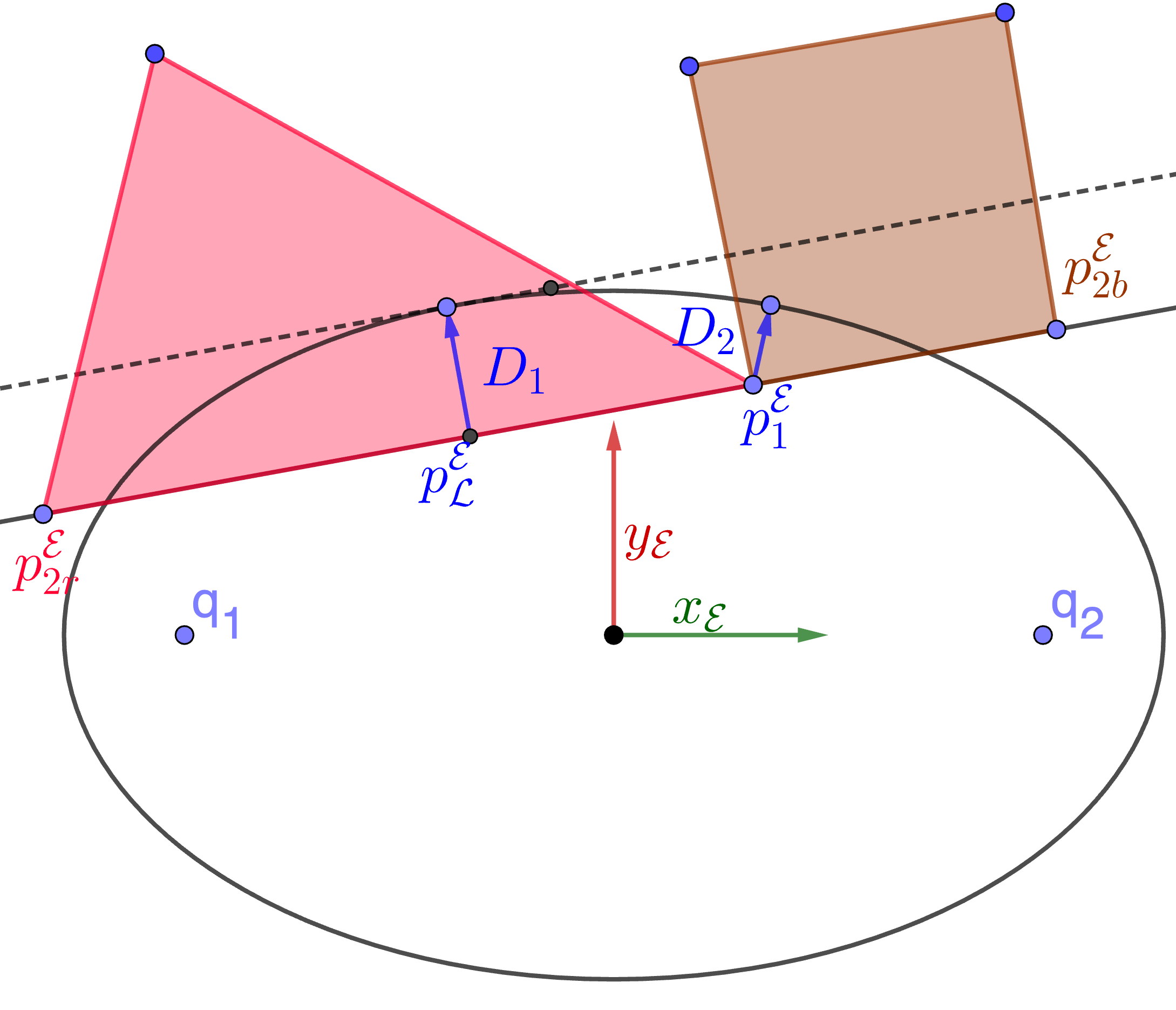}}

    \caption{ (\ref{fig:EllipseConstraintExample}) The black line is the planned trajectory, $q_1$ and $q_2$ are two co-observed locations where the robots are expected at given times $t_1$ and $t_2$, dashed lines show a possible trajectory of a compromised robot during the unobserved period. The axis of global frame $\cF$ and canonical frame $\cF_{\cE}$ are shown as $x_{\cF}, y_{\cF}$ and $x_{\cE}, y_{\cE}$ respectively. 
    (\ref{fig:Ellipse-to-point}) For point-ellipsoid constraint, $q_{avoid}$ is projected to the areas outside the ellipsoid to $q_{p}$. 
    (\ref{fig:Ellipse-to-plane}) For plane-ellipsoid constraint, the projection is simplified to the point-ellipse constraint that projects point $p_{L}$ outside the ellipse to $p_{t}$.
    (\ref{fig:Ellipse-to-safezone}) For convex-polygon-ellipsoid constraint, the projection is either a plane-ellipse constraint $D_{1}$ (for the red region) or a point-ellipse constraint $D_{2}$ (for the brown region).}
    \label{fig:Reachability_full}
  \end{figure}

  The condition that a reachability ellipsoid $\cE(q_1,q_2)$ does not intersect with any forbidden region is sufficient to secure the trajectory of the robot between $q_1$ and $q_2$ according to \cref{def:secured-plan}: any deviation to a point $p_o$ outside $\cE(q_1,q_2)$ will cause the robot to miss a potential observation at $p_2$, $t_2$. \cref{def:secured-plan} can then be restated as follows:
\begin{definition}\label{rmk:revised-security}
  A multi-robot trajectory is secured against plan-deviation attacks if there exists a co-observation plan such that the reachability region between each consecutive co-observation does not intersect with any forbidden regions.
\end{definition}

\subsection{Transformation to canonical coordinates}\label{sec:rotation2Standard}
\newcommand{\oFE}{o}
To use the definition of reachability ellipsoid from the section above as computational constraints for ADMM, we first apply a differentiable rigid body transformation to reposition the ellipse $\cE$ from the global frame $\cF$ to a canonical frame $\cF_\cE$, where the origin of $\cF_\cE$ is at the ellipsoid's center $\oFE = \frac{1}{2}(q_1+q_2)$, and the first axis of $\cF_\cE$ is aligned with the foci (see \cref{fig:EllipseConstraintExample} for an illustration). 
From this definition, a unit vector along the first axis of the ellipsoid in the frames $\cF$ and $\cF_\cE$ is given by $\nu_\cF = \frac{q_2-q_1}{\norm{q_2-q_1}}$ and $\nu_\cE =[1,0,0]^T$, respectively.
The full transformation from coordinates $q^\cE\in\cF_\cE$ to $q\in\cF$, and its inverse, are then parametrized by a rotation $R^\cF_\cE$ and a translation $o^\cF_\cE$ as (we drop the subscript and superscript from $R^\cF_\cE$ and $o^\cF_\cE$ to simplify the notation):
\begin{equation}\label{eq:transformations}
  q=R\Eframe{q}+\oFE,\quad
  \Eframe{q}=R\transpose(q-\oFE).
\end{equation}
where the rotation matrix $R=H(\nu_\cF(q_1,q_2),\nu_\cE)$ is a \emph{Householder rotation} $H$ (a differentiable linear transformation describing the minimal rotation between the two vectors, see \cref{sec:householder} for details).  To simplify the notation, in the following, we will consider $H$ to be a function of $q_1,q_2$ directly, i.e. $H(q_1,q_2)$. More details on \eqref{eq:transformations} are given in \cref{apx:transformation}.

  Reachability constraints are formulated with respect to different types of forbidden regions a point, a plane, a segment, and a convex polygon.

\subsection{Point-ellipsoid reachability constraint}\label{sec:ellipsoid-point}
 For a forbidden region in the shape of a single point $q_{avoid}$ (\cref{fig:Ellipse-to-point}), the constraint is written as:
\begin{constraint}[Point-ellipsoid reachability constraint]
\begin{align}
D(\vq) &=  \begin{cases}
      \pi_{p\cE}(\vq) - q_{avoid} & q_{avoid} \in \cE, \\
      {0} & \textrm{otherwise}.
    \end{cases} \label{eq:point_ellipsoid_constraint}\\
  \sZ &= \{\vq\in\mathbb{R}^{nm}:\norm{D_p(\vq)} = 0 \},\\
   \Pi_\sZ(\vz) & = 0, 
\end{align}
\end{constraint}

where $\pi_{p\cE}(q_{avoid};q_1,q_2,a) = q_p$ is a projection function that returns a projected point $q_p$ of $q_{avoid}$ outside the ellipse, i.e., as the solution to
\begin{equation}\label{eq:ellipse-point-projection}
\pi_{p\cE}=\argmin_{q_p\in\cE^c} \quad \norm{q_{avoid}-q_p}^2,
\end{equation}
where $\cE^{c}$ is the set complement of region $\cE$.

For cases where $q_{avoid} \notin \cE(q(t_2),q(t_1),r)$, $\pi_{p\cE}(q_{avoid}) = q_{avoid}$. And for cases where $q_{avoid} \in \cE(q_1,q_2,r)$, $D(q)$ needs to be projected to the boundary of the ellipse. 
In the canonical frame $\cF_{\cE}$, the projected point $q^{\cE}_{p}$ can be written as:
  \begin{equation}
  	q^\cE_{p} = {(I+sQ)}\inverse q^\cE_{avoid} = S q^\cE_{avoid},
  \end{equation}
where $s$ can be solved as the root of the level set \eqref{equ:standard-ellipse}:
  \begin{equation}
    {q^\cE_p}\transpose Q q^\cE_p-1={q^\cE_{avoid}}\transpose Q'(s) q^\cE_{avoid} -1=0,
  \end{equation}
  where
  \begin{equation}
    Q'(s) =S\transpose Q S = \diag\left(\frac{a^{2}}{(s+a^{2})^2}, \frac{b^{2}}{(s+b^{2})^2},\frac{b^{2}}{(s+b^{2})^2}\right).
  \end{equation}
 Detailed methods for computing $s$ can be found in \cite{yang2021multi,yang2020multi,eberly}.

The point-to-ellipse projection function in $\cF$ is then:
  \begin{multline}\label{equ:Point2EllipseProjection}
    \pi_{p\cE}(q)= R^{-1}(q(t_1),q(t_2))q^\cE_p+o \\
    = R^{-1}(q(t_1),q(t_2))Sq^\cE_{avoid}+o\\
    = R^{-1}S R ( q_{avoid}- o)+o.
  \end{multline}
  
In our derivations, we consider only the 3-D case ($m=3$); for the 2-D case, let $P=\bmat{I & 0}\in\real{2\times 3}$: then $\pi_{p\cE}^{\textrm{2D}}=P\pi_{p\cE}^{\textrm{3D}}(P\transpose q_{avoid}; P\transpose q_1, P\transpose q_2,a)$.

  \begin{proposition}\label{prop:Ellipse2PointDiff}
    The differential of the projection operator $\pi_{p\cE}(q_{avoid}; q_1,q_2,a)$ with respect to the foci $q_1,q_2$ is given by the following (using $q$ as a shorthand notation for $q^\cE_{avoid}$)
    \begin{multline}
      \partial_{\left[\begin{smallmatrix}q_1\\q_2\end{smallmatrix}\right]} \pi_{p\cE}=-2 H [ SH(q-o)]_{\times}U   \\
      + \left( (q\transpose \partial_sQ' q)^{-1} H^{-1} Q' q q\transpose  (4Q' H[q - o]_\times U \right. \\
      \left.+ 2Q' H\partial_q o - \partial_b Q' q q \partial_q b) -  s H^{-1} S^2 \partial_b Q q \partial_q b\right) \\
      -2H^{-1} S H[q-o]_{\times} U  + (H^{-1}SH -I)\partial_q o.
    \end{multline}
  \end{proposition}
  \begin{proof}
  See \cref{proof:Ellipse2PointDiff}.
  \end{proof}
  The differential of $D_p$ is the same as the one for $\pi_{p\cE}$.

\subsection{Plane-ellipsoid reachability constraint}\label{sec:ellipsoide-plane}

For a hyperplane shaped forbidden region $\mathcal{L}(q) = \{q\in\mathbb{R}^m: \vn\transpose q = d\}$ (\cref{fig:Ellipse-to-plane}), the reachability constraint is $\mathcal{L} \cap \mathcal{E}(q_1,q_2,a) = \emptyset$. When transformed into the canonical frame, the hyperplane can be written as $\Eframe \cL(\Eframe{q}) = \{  \Eframe{q}\in\mathbb{R}^m:  \vn_\cE\transpose \Eframe{q} = d_{\cE}\}$, with $\vn_\cE = H(q_1,q_2)\vn, \quad d_\cE =-\vn\transpose o + d$.

For every $\Eframe \cL(\Eframe{q})$, there exist two planes that are both parallel to $\cL$ and tangential to the ellipse (i.e. resulting in a unique intersection point \cref{fig:Ellipse-to-plane}), $\cL_1^{\cE} = \{\Eframe{q}\in\mathbb{R}^m: \vn_\cE \transpose \Eframe{q} =  d_{\cE t} \}$ and $\cL_2^{\cE} = \{\Eframe{q}\in\mathbb{R}^m: \vn_\cE \transpose \Eframe{q} = - d_{\cE t} \}$. The intersection point can be written as:
\begin{equation}
    p^\cE_{t_{1}} = \frac{ d_{\cE t} Q^{-1}   n_\cE}{ n_\cE \transpose Q^{-1}  n_\cE} = \frac{Q^{-1} n_\cE}{ d_{\cE t}},\quad  p^\cE_{t_{2}} = -  p^\cE_{t_{1}},
\end{equation}
where $d_{\cE t} = \sqrt{ n_\cE\transpose Q^{-1} n_\cE}$; intuitively, $d_{\cE t}$ can be treated as a distance between the tangent plane $\cL_1^\cE$ (or $\cL_2^\cE$) and the origin (i.e., the center of the ellipse $\cE$). The concept of \emph{tangent interpolation point}s are introduced to characterize the relationship between the plane and the ellipsoid.

\begin{definition}
  The \emph{tangent interpolation point} $p^\cE_{\cL} \in \mathcal{L}$ between the plane $\cL$ and the ellipsoid $\cE$ is defined on the plane by 
    \begin{equation}\label{equ:ProjectPoint}
      p^\cE_{\cL} = \frac{ d_\cE Q^{-1} n_\cE}{ n_\cE\transpose Q^{-1} n_\cE}.
    \end{equation}
  Intuitively, the point $p^\cE_\cL$ is the closest point on $\cL$ to either $p^\cE_1$ or $p^\cE_2$.
  Note that when  $d_\cE=d_{\cE t}$ or $d_\cE=-d_{\cE t}$, $p^\cE_{\cL}=p^\cE_{t_1}$ or $p^\cE_{\cL}=p^\cE_{t_2}$, respectively. When $ d_\cE \in [- d_{\cE t},  d_{\cE t}]$, the plane $\cL$ and the ellipsoid $\cE$ have at least one intersection, thus violating our desired reachability constraint. 
\end{definition}

With these definitions, the constraint can be written as:
\begin{constraint}[Plane-ellipsoid reachability constraint]
\begin{align}
D(\vq) &= H^{-1}(q_{1},q_2)\pi^\cE_{\vn_\cE}(\vq)+o, \label{eq:plane_ellipsoid_constraint}\\
  \sZ &= \{\vq\in\mathbb{R}^{nm}:\norm{D_{\vn}(\vq)} = 0 \},\\
   \Pi_\sZ(\vz) & = \overrightarrow{0}, 
\end{align}
\end{constraint}
where $\pi^\cE_{\vn_\cE}(q)$ is the projection operator defined as:
  \begin{equation}\label{equ:plane2ellipse}
    \pi^\cE_{\vn_\cE}(\vq) =\begin{cases}
      p^\cE_{t_1}-p^\cE_{\cL} & \textrm{ if } d_\cE \in [0, d_{\cE t}], \\
      p^\cE_{t_2}-p^\cE_{\cL} &  \textrm{ if } d_\cE \in [- d_{\cE t},0), \\
      {0} & \textrm{otherwise}.
    \end{cases}
  \end{equation}

  \begin{proposition}\label{prop:dpi_ne_dt}
    The differential of the projection function $\pi^\cE_{\vn \cE}(\vq)$ with respect to the foci $q_1$ and $q_2$ is given by:
    \begin{equation}
      \partial_q \pi^\cE_{\vn \cE}(q) = \begin{cases}
        \partial_q p^\cE_{t1}-\partial_q p^\cE_{\cL} &  d_\cE \in [0, d_{\cE t}], \\
        \partial_q p^\cE_{t2}-\partial_q p^\cE_{\cL} &  d \in [- d_{\cE t},0), \\
        {0} & otherwise.
      \end{cases}
    \end{equation}
    where
    \begin{multline}
      \partial_q p_\cL =   (-\frac{d_{\cE t} n\transpose \partial_q o -2d_\cE \partial_q d_{\cE t}}{d_{\cE t}^3} )Q^{-1}n_\cE \\
      + \frac{d_\cE\partial_b Q^{-1} n_\cE \partial_q b -  2d_\cE Q^{-1}H[n]_\times U}{d_{\cE t}^2},
    \end{multline}
    \begin{equation}
      \partial_q p_{1} =  -\frac{Q^{-1}n_\cE \partial_q d_{\cE t} }{d_{\cE t}^2} 
      + \frac{\partial_b Q^{-1}n_\cE \partial_q b -  2Q^{-1}H[n]_\times U }{d_{\cE t}}.
    \end{equation}
  \end{proposition}
  \begin{proof}
  See \cref{proof:dpi_ne_dt}.
  \end{proof}
 The differential of (\ref{eq:plane_ellipsoid_constraint}) can be written as:
  \begin{equation}
    \partial_q {D_{\vn\cE}} = -2 H \cross{\Eframe \Pi_{\vn \cE}}  M + H^{-1}\partial_q \Eframe \Pi_{\vn \cE}.
  \end{equation}
\subsection{Convex-polygon-ellipse reachability constraint}\label{sec:ellipse-region-constraint}  
The reachability constraint for a convex polygon is treated as a union of reachability constraints for each individual segments that define the hyperplane. We define
\begin{constraint}[Convex-polygon-ellipsoid reachability constraint]\label{constraint:polygon-ellipsoid}
\begin{align}
D(\vq) &= \bmat{D_{seg1}(\vq)\\D_{seg2}(\vq)\\ \vdots} \label{eq:region_ellipsoid_constraint}\\
  \sZ &= \{\vq\in\mathbb{R}^{nm}:\norm{D(\vq)} = 0 \},\\
   \Pi_\sZ(\vz) & = 0, 
\end{align}
\end{constraint}
where $D_{seg}$ are the constraint functions for all line segments that define the convex polygon region. 
For hyperplane $\Eframe \cL(\Eframe q) = \{ \Eframe q\in\mathbb{R}^m:  \vn_\cE\transpose \Eframe q = d_{\cE}\}$ with endpoints $\Eframe {p_1}$ and $\Eframe{p_2}$; this segment is defined as:
\begin{equation}
\bmat{(\Eframe {p_1}- \Eframe{p_2})\transpose\\( \Eframe{p_2} - \Eframe{p_1})\transpose}  \Eframe p \leq \bmat{ \Eframe{p_2}\transpose\\- \Eframe{p_1}\transpose}( \Eframe{p_1}- \Eframe{p_2}), \quad \vn_\cE\transpose \Eframe q = d_{\cE}.
\end{equation}
When the \emph{tangent interpolation point} $\Eframe{p_{\cL}}$ stays within the segment (i.e. red region in \cref{fig:Ellipse-to-safezone}), the constraint is a plane-ellipse constraint in \cref{sec:ellipsoide-plane}. Otherwise (i.e. brown region in \cref{fig:Ellipse-to-safezone}), the constraint is a point-ellipse constraint in \cref{sec:ellipsoid-point}.

\begin{constraint}[Line-segment-ellipsoid reachability constraint]
\begin{align}
D_{seg}(\vq) &=  \begin{cases}
D_{p_{1}}(\vq) & (\Eframe{p_1}-\Eframe{p_2})\transpose (\Eframe{p_{\cL}}-\Eframe{p_2})<0\\
D_{p_{2}}(\vq) & (\Eframe{p_2}-\Eframe{p_1})\transpose (\Eframe{p_{\cL}} -\Eframe{p_1})<0\\
D_\Eframe{p_{\cL}}(\vq) & \textrm{otherwise}
\end{cases} \label{eq:segment_ellipsoid_constraint}\\
  \sZ &= \{\vq\in\mathbb{R}^{nm}:\norm{D_{seg}(\vq)} = 0 \},\\
   \Pi_\sZ(\vz) & = 0, 
\end{align}
\end{constraint}
where $D_{p_{1}}$ and $D_{p_{2}}$ are the point-ellipsoid constraint projection function~\eqref{eq:point_ellipsoid_constraint} with respect to $\Eframe{p_{1}}$ and $\Eframe{p_{2}}$, and $D_\Eframe{p_{\cL}}$ is the plane-ellipsoid constraint~\eqref{eq:plane_ellipsoid_constraint} with respect to frame $\Eframe {\cL}$.
This constraint needs to be supplemented with a convex obstacle constraint for the polygon (i.e. keep foci waypoints outside the region, introduced in \cite{yang2020multi}) to prevent cases where the ellipse is a subset of the region.

\subsection{Secured planning results and limitation}\label{sec:ADMM-simulation}

We employ the Attack-Proof MAPF (APMAPF) solver \cite{wardega2019resilience} on an 8 by 8 grid world with a similar setup to generate a MAPF plan with a co-observation schedule in a grid-world application (\cref{fig:Grid-example-application}). This result is transformed to a continuous configuration space and serves as the initial trajectory input for the ADMM solver with an additional task function targeting the map exploration task. Co-observation schedules are set up using the APMAPF algorithm for two forbidden regions. Reachability constraints are added to ensure an empty intersection between all robots' reachability regions during co-observations and the forbidden regions. It is important to emphasize that an APMAPF solution does not guarantee existence, nor does it ensure a successful transition to the continuous configuration space. In this case, while an APMAPF solution may be found, the secured, attack-proof solution becomes infeasible in the continuous setting because the robots' mobility is no longer restricted to adjacent grids. Additional security measures, such as stationary security cameras or surveillance robots, need to be incorporated. For instance, in this scenario, we deploy a security camera as an additional security measure to observe agent 3 at time $8$ to ensure security.

The simulation result, shown in Fig.~\ref{fig:ReachabilitySimulation}, displays reachability regions as black ellipsoids, demonstrating empty intersections with region 2 and 3. Explicit constraints between reachability regions and obstacles are not activated, assuming basic obstacle avoidance capabilities in robots. The intersections between obstacles and ellipsoids, as observed between agent 3 and region 1, are deemed tolerable. All constraints are satisfied, and agents have effectively spread across the map for optimal exploration tasks.

\begin{figure}
  \centering
  \subfloat[Grid world result\label{fig:Grid-example-application}]{\includegraphics[height = 0.47\linewidth,trim =0cm 0cm 0cm 0cm, clip,valign=t]{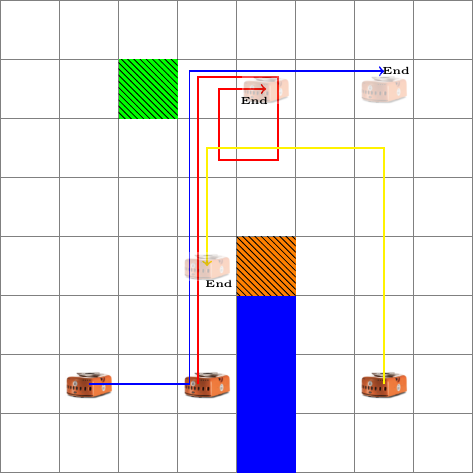}}
  \subfloat[Secured result \label{fig:ReachabilitySimulation}]{\includegraphics[height=0.47\linewidth,valign=t]{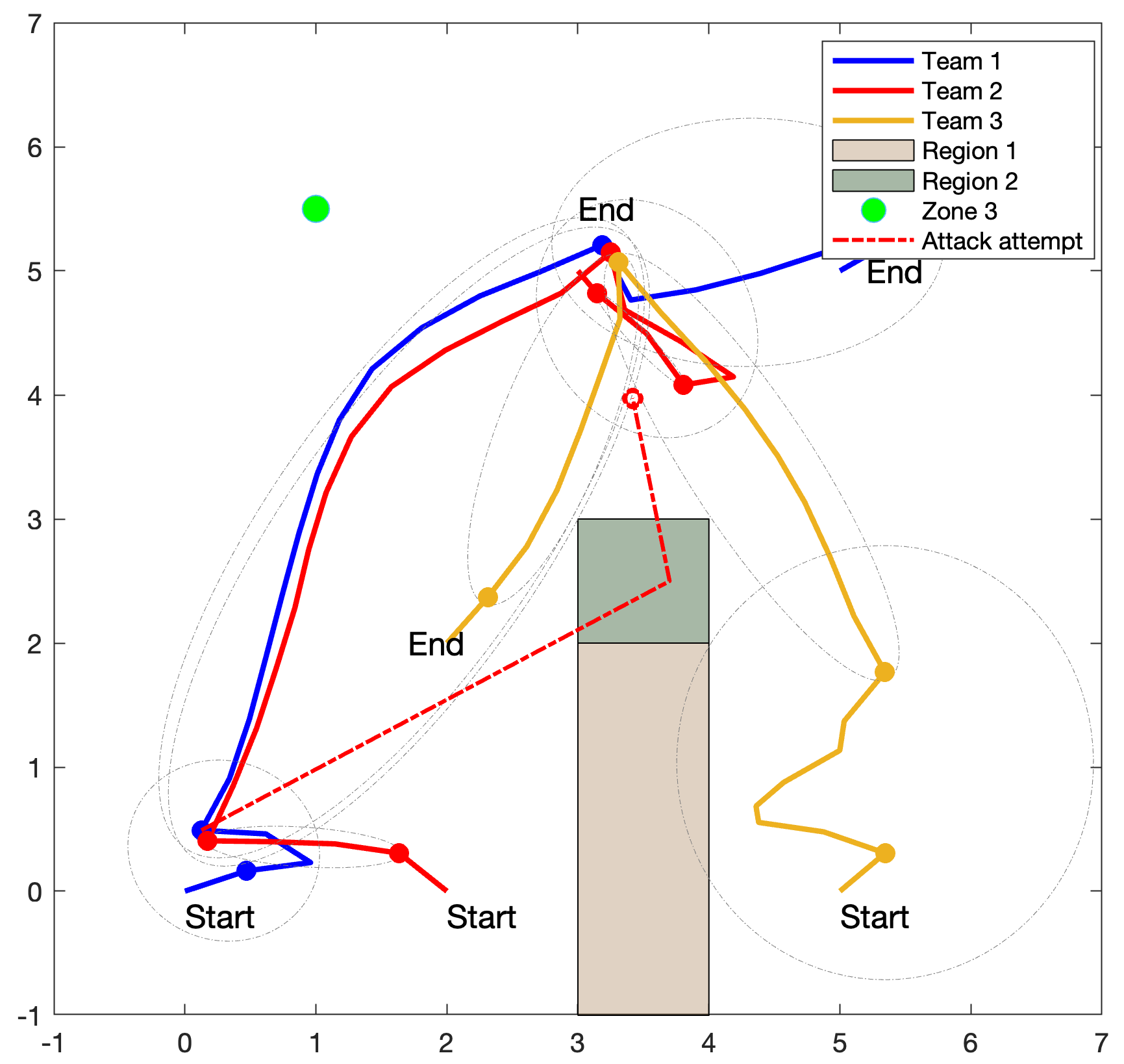}}
  \caption{ Trajectory design of a three-robot system based on the grid world result. Region 1 is an obstacle, region 2 and region 3 are forbidden regions. (\ref{fig:Grid-example-application}) Result of the APMAPF algorithm in $8\times8$ grid-world without the map exploring objective. (\ref{fig:ReachabilitySimulation}) The secured trajectory with the incorporation of co-observation generated by APMAPF and additional reachability constraints. Attack attempts into the forbidden region (red dashed line) will cause team 1 miss the next scheduled co-observation.}
  \label{fig:example-application}
\end{figure}

\subsubsection{Limitations}\label{sec:reachability-discussion}
Our solution demonstrates the potential of planning with reachability and co-observation to enhance the security of MRS. However, two primary challenges need to be addressed. Firstly, achieving a co-observation and reachability-secured plan is not always feasible, particularly when obstacles or restricted regions separate the robots, preventing them from establishing co-observation schedules or finding reachability-secured paths. For example, agent 3 in~\cref{fig:ReachabilitySimulation} requires additional security measures to create secured reachability areas. Secondly, security requirements can impact overall system performance, as illustrated by the comparison between~\cref{fig:SecurityBreak} and \cref{fig:ReachabilitySimulation}. The introduction of security constraints resulted in the top left corner remaining unexplored by any robots. This trade-off between security and system performance is particularly significant as system performance is a key factor in the decision of MRS deployments. These challenges are further addressed in \cref{sec:cross-trajectory}, ensuring the effective integration of reachability and co-observation in securing MRS.

\section{Cross-trajectory co-observation planning}\label{sec:cross-trajectory}
To address the feasibility and performance trade-offs, we propose to form \emph{sub-teams} on each route, and setup additional co-observations both within the sub-team and across different sub-teams. These \emph{cross-trajectory co-observations} allow sub-teams to obtain trajectories that are closer to the optimal (as shown in \cref{fig:cross-traj-comparison-set}), because they do not require the entire \emph{sub-team} to meet with other teams for co-observations, thus providing more flexibility.
\subsection{Problem overview}
We start with an unsecured MRS trajectory with $n_p$ routes $\{q_p\}_{p=1}^{N_p}$ (the ADMM-based planner in \cref{sec:ADMM-planning} without security constraints is used, but other planners are applicable). Here, the state $q_p(t), p \in\{\cI_0,\ldots,\cI_{p}\}$ represents the reference position of the $p$-th sub-team at time $t\in\{0, \dots, T\}$. Introducing redundant robots into the MRS, we assume a total of $n > n_p$ robots are available and organized into \emph{sub-teams} through a time-varying partition $\cI(t)=\union_p \cI_p(t)$ where robots in each \emph{sub-team} $\cI_p$ share the same nominal trajectory. To ensure the fulfillment of essential tasks, at least one robot is assigned to adhere to the reference trajectory. The remaining redundant $n_p-n$ robots are strategically utilized to enhance security. These robots focus on co-observations either within their respective sub-teams, or when necessary, deviate and join another sub-team to provide necessary co-observations. The objective of this problem is to formulate a strategy for this new co-observation strategy, named \emph{cross-trajectory co-observation}, such that the resulting MRS plan meets the security in \cref{rmk:revised-security} while minimizing the number of additional robots required. 

Our strategy is based on the following concept:
\begin{definition}
The $i$-th \emph{checkpoint} $v_{pi}=(q_{p},t_{i})$ for sub-team $p$ is defined as a location $q_{pi}=q_{p}(t_i)$ and time $t_{i}$. It represents either the first or last instance where a robot is co-observed by a teammate, and serves as a point where the robot can leave or join a new team.
\end{definition}
For simplicity, let $\cI_{v_{i}}$ denote the sub-team to which $v_{i}$ belongs. 
To ensure the security of the reference trajectory, the reachability region between consecutive checkpoints must avoid intersections with forbidden regions. This requirement can be formally stated as:
\begin{remark}\label{rmk:checkpoints}
  A set of checkpoints $V_{p}=\{ v_{p0}, \dots ,v_{pT}\}$ (arranged in ascending order of $t_{v_{pi}}$) can secure the reference trajectory for sub-team $p$, if $\mathcal{E}(q_{v_{pi}}, q_{v_{p(i+1)}}, t_{v_{pi}},t_{v_{p(i+1)}}) \intersect F = \emptyset$ for every $i$, where $F$ is the union of all forbidden regions.
  \end{remark}

\begin{definition}
  The planned trajectory $\{q_p\}_{p=1}^{N_p}$ is represented as a directed \emph{checkpoint graph} $G=(V, E)$, where $V = (\union_p V_p) \union V_c$ is the set of vertices representing waypoints on the planned trajectory and $E = E_{t} \union E_{c}$ is the set of edges representing feasible paths connecting $V$. The components of the checkpoint graph are defined as follows:
  \begin{description}
    \item[Checkpoints $V_p$] where additional robots are required for security through co-observation.
    \item[Trajectory edges $E_t$] represent the original planned path of the primary robot.
    \item[Connection vertices $V_c$] represent non-checkpoint locations linked via $E_c$.
    \item[Cross-trajectory edges $E_c$] represent feasible paths where robots deviate from the original trajectory to join other teams or provide co-observation support, at least one endpoint be a checkpoint. 
  \end{description}
\end{definition}

We assume that the planned trajectory for each sub-team has a robot operating on it. The co-observation planning problem is then transformed into a network multi-flow problem by modeling the additional robots as flows traveling through $G$. These flows either follow the planned trajectory via trajectory edges $E_t$ or switch teams via cross-trajectory edges $E_c$ to ensure valid and secured paths. Additionally, they are required to visit every checkpoint $V_p$ to meet the security requirements introduced in~\ref{rmk:checkpoints}. This formulation reduces the co-observation problem to finding valid flow patterns in $G$. The problem is then formulated into linear objectives and constraints, enabling the use of Mixed-Integer Linear Programming (MILP) techniques to efficiently compute feasible solutions. 

This section presents the two components of our approach: constructing the checkpoint graph based on unsecured multi-robot trajectories and the formulation and solution of the network multi-flow problem.

\subsection{Rapidly-exploring Random Trees}
To find edges between different nominal trajectories, we use the \rrtstar{}~\cite{karaman2010incremental} algorithm to find paths from a waypoint on one trajectory to multiple destination points (i.e., reference trajectories of other sub-teams, in our method). As an optimal path planning algorithm, \rrtstar{} returns the shortest paths between an initial location and points in the free configuration space, organized as a tree. We assume that the generated paths can be traveled in both directions (this is used later in our analysis). 
Key function from \rrtstar{} is used during the construction of the checkpoint graph, specifically
\begin{description}
\item[$\texttt{Cost}(v)$] This function returns a total travel distance to the unique path from the initial position to $v$. 
\end{description}

Our objective here is to ascertain the existence of feasible paths instead of optimizing specific tasks. While the ADMM-based solver \cref{sec:ADMM-planning} presented earlier offers a broader range of constraint handling, RRT*'s efficient exploration of the solution space, coupled with its ability to incorporate obstacle constraints, makes it a fitting choice for building the checkpoint graph.

\subsection{Checkpoint graph construction}\label{sec:security-checkpoint}

In this section, we define and search for security checkpoints and how to use \rrtstar{} to construct the checkpoint graph $G_{q}$. We use a heuristic approach (\cref{alg:graph}) to locate the checkpoints on given trajectories. \cref{alg:graph} is applied independently to each sub-team to generate their initial checkpoints.” While an optimal solution would likely be NP-hard, this approach works well enough for our purpose.

\begin{algorithm}
	\caption{Checkpoint Graph Initialization for Team $\cI_p$}\label{alg:graph}
  \begin{algorithmic}[1]
	 \State \textbf{Initialization:} Set $t_{0} \leftarrow 0$ and $t_{3} \leftarrow T$, and add $(q_{p}(t_0),t_0)$ and $(q_{p}(t_3),t_3)$ to $V_{p}$ as \emph{start} and \emph{end vertices}.
	\While{$\mathcal{E}(q_{p}(t_{0}),q_{p}(t_3),t_{0},t_3) \cap F \neq \emptyset$ \textbf{and} $t_{0} \leq t_3$}
		\State \textbf{Forward Search:} Find the largest $t_{1} \in \{t_{0}+1, \dots, t_3\}$ such that $\mathcal{E}(q_{p}(t_{0}),q_{p}(t_{1}), t_{0},t_{1}) \cap F = \emptyset$. Once found, add $(q_{p}(t_{1}),t_{1})$ to $V_{p}$.
		\State \textbf{Backward Search:} Find the smallest $t_2\in\{t_{1}, \dots, t_3-1\}$ such that $
		\mathcal{E}(q_{p}(t_3),q_{p}(t_2), t_3,t_2) \cap F = \emptyset.$ Once found, add $(q_{p}(t_2),t_2)$ to $V_{p}$.
		\State \textbf{Update Indices:} Set $t_{0} \leftarrow t_{1}$ and $t_3 \leftarrow t_2$.
	\EndWhile
	\end{algorithmic}
	\end{algorithm}
\begin{figure}[h]
	\centering
    \subfloat{\includegraphics[width=0.47\linewidth,valign=c]{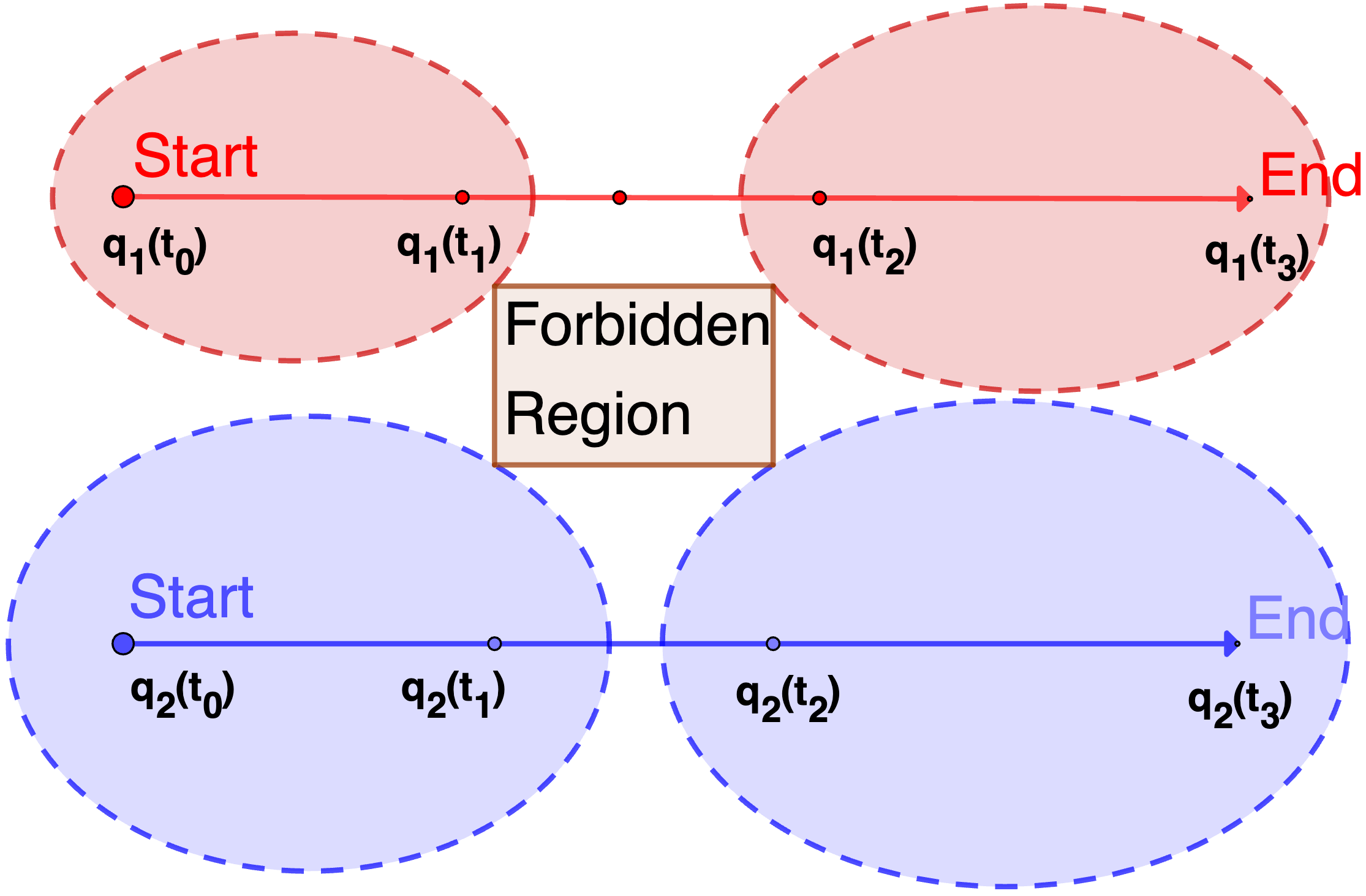}}
    \subfloat{\includegraphics[width=0.47\linewidth,valign=c]{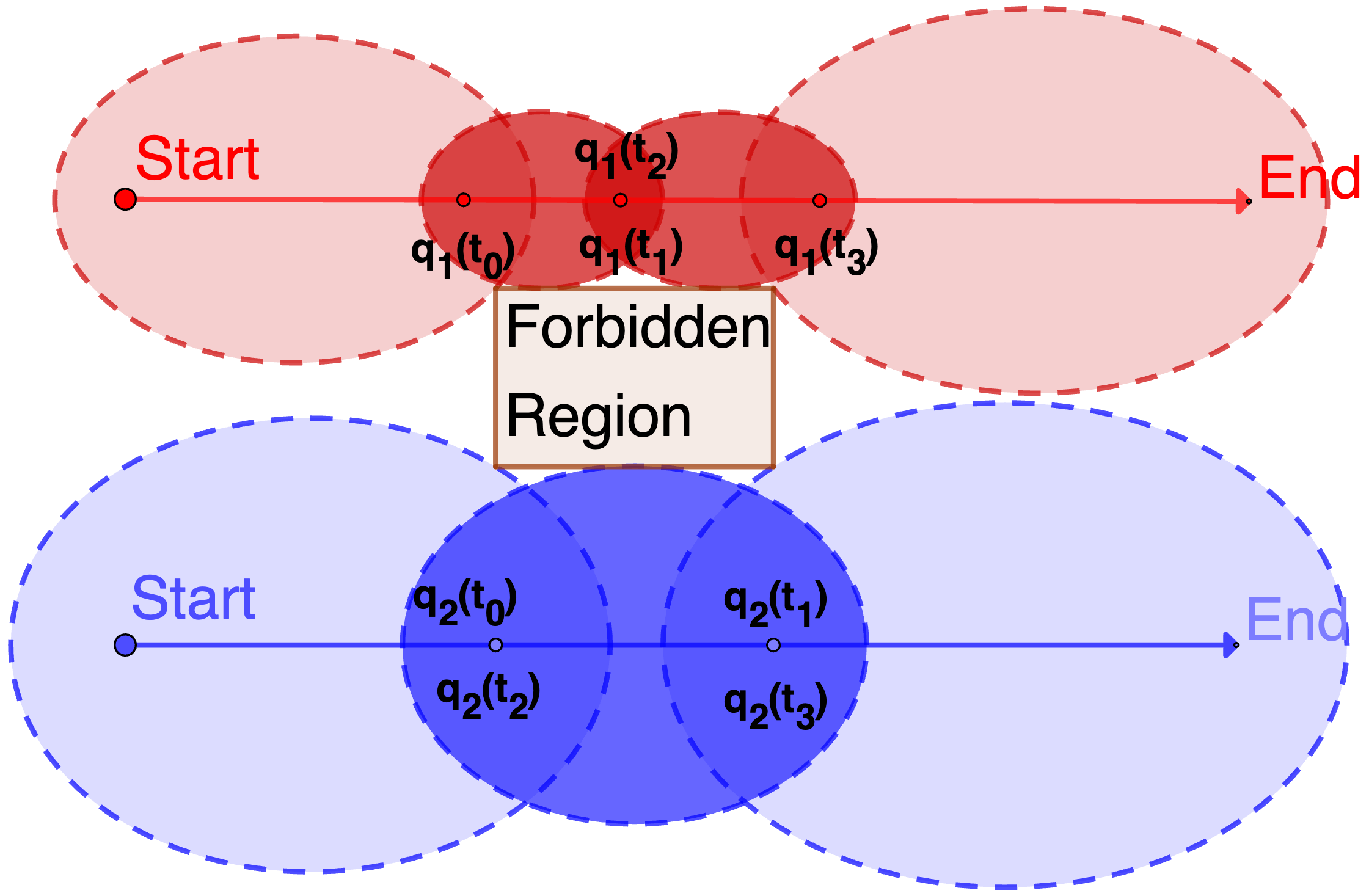}}
    \caption{ Example of \cref{alg:graph}. Start at both ends, each trajectory finds the largest reachability region avoiding forbidden areas, then repeats the process on remaining segments to build secure checkpoints.}\label{fig:checkpoint-generate}
\end{figure}

\subsubsection{Cross-trajectory edges}\label{sec:cross-traj-edges}
To enable co-observation across different sub-teams at checkpoints, we search for available connection paths (\emph{cross-trajectory edges}) between checkpoints on different trajectories, allowing robots to deviate from one sub-team to perform co-observation with a different sub-team. 

Cross-trajectory edges $E_c = (v_1, v_2)$ define viable paths between two reference trajectories, where $\cI_{v_1}\neq\cI_{v_2}$ and at least one of $v_1 $ and $v_2$ corresponds to a security checkpoint $\union_p V_p$. The cross-trajectory edges must also adhere to reachability constraints $\mathcal{E}(q_{v_1}, q_{v_2}, t_{v_1},t_{v_2}) \intersect F = \emptyset$ to ensure that no deviations into forbidden regions can occur  while switching trajectories.

\begin{figure}[htbp]
    \centering
  \subfloat[Cross-trajectory edge generation\label{fig:two-edges}]{\includegraphics[width=0.45\linewidth,valign=c]{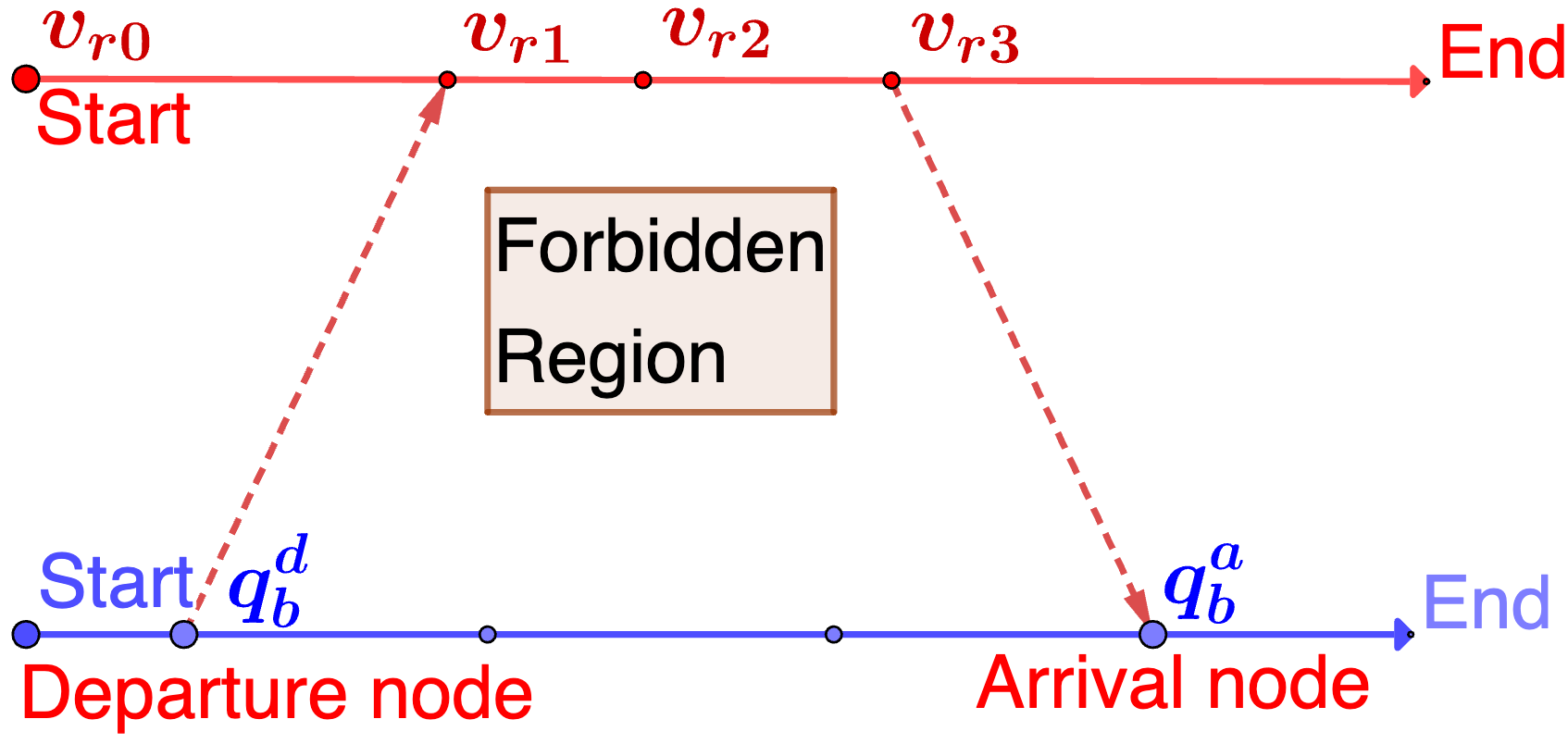}}
  \subfloat[Full checkpoint graph\label{fig:security-graph-generate}]{\includegraphics[width=0.45\linewidth,valign=c]{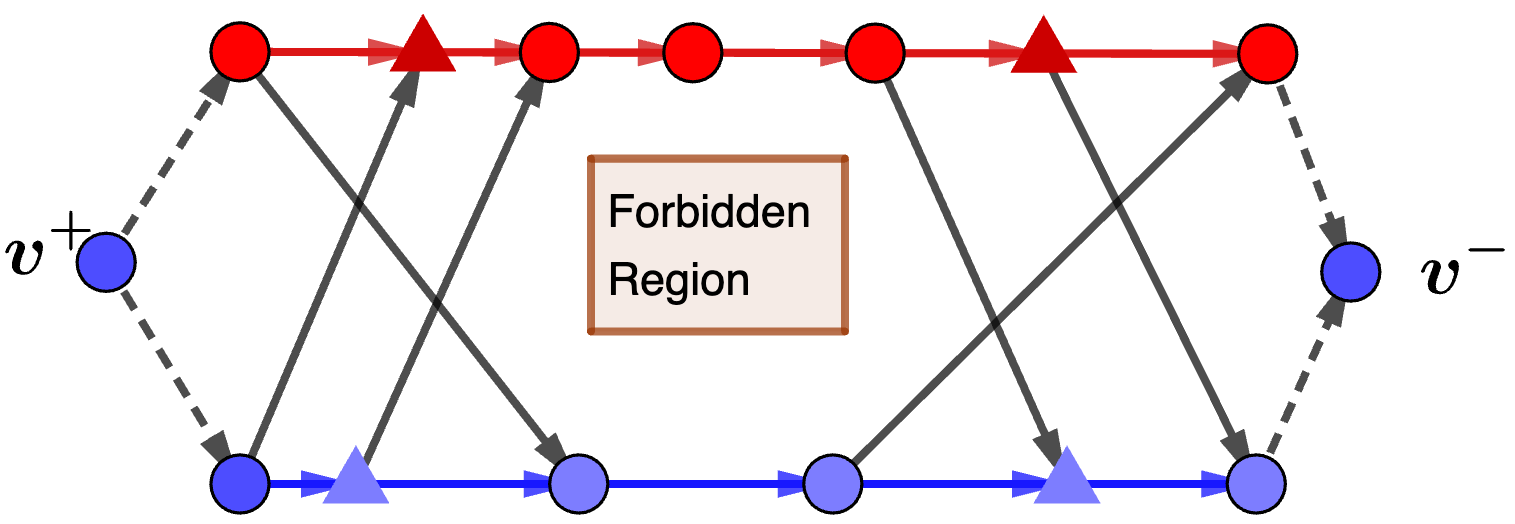}}
  \caption{ (\ref{fig:two-edges}) The latest departure node $q^{d}_{b}$ found for $v_{r1}$ and the earliest arrival node $q^{a}_{b}$ found for $v_{r3}$. (\ref{fig:security-graph-generate}) A full checkpoint graph, where circles are checkpoints $V_p$, triangles are connection nodes $V_c$, colored arrows (red and blue) are trajectory edges $E_t$ and black arrows are cross-trajectory edges $E_c$.}
\end{figure}

We first find trees of feasible paths between each security checkpoint and all other trajectories using position information alone (ignoring, for the moment, any timing constraint). 
More precisely, for each checkpoint $v_p \in \union_p V_p$, we use \rrtstar~ to find all feasible, quasi-optimal paths from $q_{v_p}(t_{v_p})$ to all the waypoints $\{q_{r}(t_{{r}_{i}})\}$ on the reference trajectories of all other sub-teams $\cI_{r} = \cI \setminus \cI_{p}$.
Then, to prune these trees, we consider the time needed to physically travel from one trajectory to the other while meeting other robots at the two endpoints. This is done by calculating the minimal travel time $t_{\textrm{path}}=\texttt{Cost}(q)/v_{max}$ for a robot to traverse each path. Two types of connecting nodes can be found.

\begin{description}
\item[Arrival nodes] are waypoints on $\{q_r\}$ where robots from sub-team $\cI_p$, deviating at $v_p$, can meet with sub-team $\cI_r$ at $(q_r(t_{r_i}),t_{r_i})$. For these, \rrtstar{} must have found a path from $v_p$ to $q_{r}(t^{a}_r)$ with $t_{p}+t_{\textrm{path}}<t_{{r}_{i}}$, and $\mathcal{E}(q_{r}(t_{{r}_{i}}),q_{p},t_{{r}_{i}},t_{p}) \intersect F = \emptyset$.

For each trajectory $r\neq p$, we define the \emph{earliest arrival node} from $v_p$ to $v_{ea} = (q_{r}(t^{a}_r),t^{a}_r)$ as the arrival node characterized by the minimum $t^{a}_r$ discovered. 

\item[Departure nodes] are the waypoints on $\{q_r\}$ such that a robot from sub-team $\cI_r$, deviating from $(q_{r}(t_{{r}_{i}}),t_{{r}_{i}})$, can meet with robots in sub-team $\cI_p$ at $v_{p}$. For these nodes, \rrtstar{} must find a path from $v_p$ to $q_{r}(t_{{r}_{i}})$ for $v_p$ if  $t_{p}>t_{{r}_{i}}+t_{\textrm{path}}$ and $\mathcal{E}(q_{p}, q_{r}(t_{{r}_{i}}),t_{p}, t_{{r}_{i}}) \intersect F = \emptyset$. 

For each trajectory $r\neq p$, we define the \emph{latest departure node} $v_{ld}=(q_{r}(t^{d}_r),t^{d}_r)$ as the departure node characterized by the maximum $t^{d}_r$ discovered.
\end{description}

The set $V_c$ contains all the \emph{latest departure} and \emph{earliest arrival} nodes (if they are not identified as checkpoints already); the corresponding paths are added as cross-trajectory edges $E_q$. A toy example is shown in \cref{fig:two-edges}.

\subsubsection{In-trajectory edges}\label{sec:Graph-intro}
We group $V =\{ v^i_p,\dots\}_{p=1}^{N_p}$ by sub-teams and arrange them in ascending order of their timestamps $t_{v^i_{p}}$. For each sub-team, consecutive vertices $v^i_{p}$ and $v^{i+1}_{p}$ are connected by adding \emph{in-trajectory edges} $\{(v^i_{p}\rightarrow v^{i+1}_{p})\}$ to $E_q$, representing the corresponding segment of the planned trajectory. Examples are shown in \cref{fig:security-graph-generate}. 

\subsection{Co-observation planning problem}
In this section, we formulate the cross-trajectory planning problem as a network multi-flow problem, and solve it using mixed-integer linear programs (MILP). We assume that $n_p$ robots are dedicated (one in each sub-team) to follow the reference trajectory (named \emph{reference robots}). The goal is to plan the routes of the $n-n_p$ additional \emph{cross-trajectory robots} dedicated to cross-trajectory co-observations, and potentially minimizing the number of cross-trajectory robots needed. 

\begin{remark}
Note that we assign fixed roles to robots for convenience in explaining the multi-flow formulation. In practice, after a cross-trajectory robot joins a team, it is considered interchangeable and could switch roles with the reference robot of that trajectory. 
\end{remark}

To formulate the problem as a network multi-flow problem, we augment the checkpoint graph to a flow graph. A \emph{virtual source} node $v^{+}$ and a \emph{virtual source} node are added to the vertices $V= (\union_p V_p) \union V_c \union \{v^{+},v^{-}\}$. Additional directed edges from $v^{+}$ to all the start vertices, and from all end vertices to $v^{-}$ are added $E = E_{q} \union \{(v^{+},v^{0}_{p})\}_{p} \union \{( v^{T}_{p},v^{-})\}_{p}$ with $v^0_p$ and $v^T_p$ representing the start and end vertices of sub-team $\cI_p$ (dashed arrows in~\cref{fig:example-cross-traj}).

The path of a robot $k$ all starts from $v^{+}$ and ends at $v^{-}$, and are represented as a flow vector $\vf^{k} = \{ f^{k}_{ij} \}$, where $f^{k}_{ij} \in \{1,0\}$ is an indicator variable representing whether robot $k$'s path contains the edge $v_{i}\to v_{j}$. 
The planning problem can be formulated as a vertex path cover problem on $G_{q}$, i.e., as finding a set of paths $F=[\vf_{1},\dots, \vf_{\cK}]$ for cross-trajectory robots such that every checkpoint in $\union_p V_{p}$ is included in at least one path in $F$ (to ensure co-observation at every checkpoint as required by \cref{rmk:revised-security}). 

Technically, we can always create a trivial schedule that involves only co-observations between members of the same team; this, however, would make the solution more vulnerable in the case where multiple agents are compromised in the same team. In this paper we explicitly consider only the single-attacker scenario, multi-attacks can be potentially handled by taking advantage of the \emph{decentralized blocklist protocol} introduced in \cite{wardega2023byzantine}. For this reason, we setup the methods presented below to always prefer \emph{cross-trajectory co-observation} when feasible.

Finally, edges from the virtual source and to the virtual sink should have zero cost, to allow robots to automatically get assigned to the starting point that is most convenient for the overall solution (lower cost when taking cross-trajectory edges).
These requirements are achieved with the weights for edges $(v_{i},v_{j})\in E$ defined as:
\begin{equation}
	w_{i,j}=\begin{cases}
	-w_{t} & \cI_{v_{i}}=\cI_{v_{j}}, (v_{i},v_{j})\in E_{q}\\
	w_{c} & \cI_{v_{i}} \neq \cI_{v_{j}}, (v_{i},v_{j})\in E_{q}\\
	0 &  (v_{i},v_{j})\in E / E_{q} 
	\end{cases}
\end{equation}
where $w_{c} > w_{t}$. 

With the formulation, the planning problem is written as an optimization problem, where the optimization cost balances between the co-observation performance and the total number of flows (cross-trajectory robots) needed:
 \begin{subequations} \label{eq:flow-coverage-problem}
     \begin{align}
        \min_{F} &\sum^{\cK}_{k} \sum_{(+i)\in E} f^{k}_{+i} - \rho \sum^{\cK}_k \sum_{(ij)\in E} w_{ij} f^k_{ij} \label{eq:flow-cost}\\
        s.t. & \sum_{\{h:(hi) \in E\}}f^k_{hi}=\sum_{\{j:(ij) \in E\}}f^k_{ij},  \forall k,\forall v \in V_{q} / \{v^{+}, v^{-}\}  \label{eq:FlowBalanceConstraint}\\
        & \sum^{\cK}_k\sum_{\{i:(ij)\in E \}} f^k_{ij} \geq 1, \forall v_{j} \in \{V^{s}_{p}\} \label{eq:FlowCoverage}\\
        & f^k_{ij} \in \{0,1\} ,  \forall (ij)\in E ,\label{eq:SingleFLowCapacity}
     \end{align}
 \end{subequations}
where, for convenience, we used $(ij)$ to represent the edge $(v_{i},v_{j})$, and $(+i)$ to represent the edge $(v^{+},v_{i})$.

The first term in the cost \eqref{eq:flow-cost} represents the total number of robots used, and the second term represents the overall co-observation performance (defined as the total number of cross-trajectory edges taken by all flows beyond the regular trajectory edges); the constant $\rho$ is a manually selected penalty parameter to balance between the two terms. 
 The flow conservation constraint \eqref{eq:FlowBalanceConstraint} ensures that the amount of flow entering and leaving a given node $v$ is equal (except for $v^{+}$ and $v^{-}$). The flow coverage constraint \eqref{eq:FlowCoverage} ensures that all checkpoints $ \{V^{s}_{p}\}$ have been visited and can support a co-observation. The checkpoint graph $G$ is acyclic, placing this problem in complexity class $P$; thus, it can be solved in polynomial time \cite{1702662}. 

\subsection{Co-observation performance}
The network flow problem \eqref{eq:flow-coverage-problem} is guaranteed to have a solution for $\cK=N_{p}$ where all $\cK$ robots follow the reference trajectory ($f^{k}_{ij}=1, \forall \cI_{v_{i}}=\cI_{v_{j}}=k$). 
It is possible for the resulting flows to have a subset of flows $F_{e} \in F$ that is empty, i.e. $f_{ij}=0, \forall (v_{i},v_{j})\in E$, $f\in F_{e}$; these flows will not increase the cost and can be discarded from the solution.  

The constant $\rho$ selects the trade-off between the number of surveillance robots and security performance. An increase in robots generally enhances security via cross-trajectory co-observations; this also increases the complexity of the coordination across robots. To identify the minimum number of robots necessary, we propose an iterative approach in which we start with $\cK=1$ and gradually increase it until $F$ contains an empty flow, indicating the point where further additions of robots do not improve performance. 

\begin{remark} 
The value of row $\rho$ is upper-bounded such that the second term for each single flow in \eqref{eq:flow-cost} is always smaller than one, i.e., $ \rho \sum_{(ij)\in E} w_{ij} f^k_{ij}\leq 1$. Otherwise, the iteration continues indefinitely as adding additional flow introduce a negative term $\sum_{(+i)\in E} f^{k}_{+i} - \rho\sum_{(ij)\in E} w_{ij} f^k_{ij} = 1- \rho\sum_{(ij)\in E} w_{ij} f^k_{ij}<0$ which always makes the cost \eqref{eq:flow-cost} smaller.
\end{remark}

\subsection{Result and simulation}

\begin{figure}
  \centering
  \subfloat[Cross-trajectory co-observation plan \label{fig:previous_result_plan}]{\includegraphics[width=0.47\linewidth, trim = 2cm 1cm 1cm 1cm,clip,valign=b]{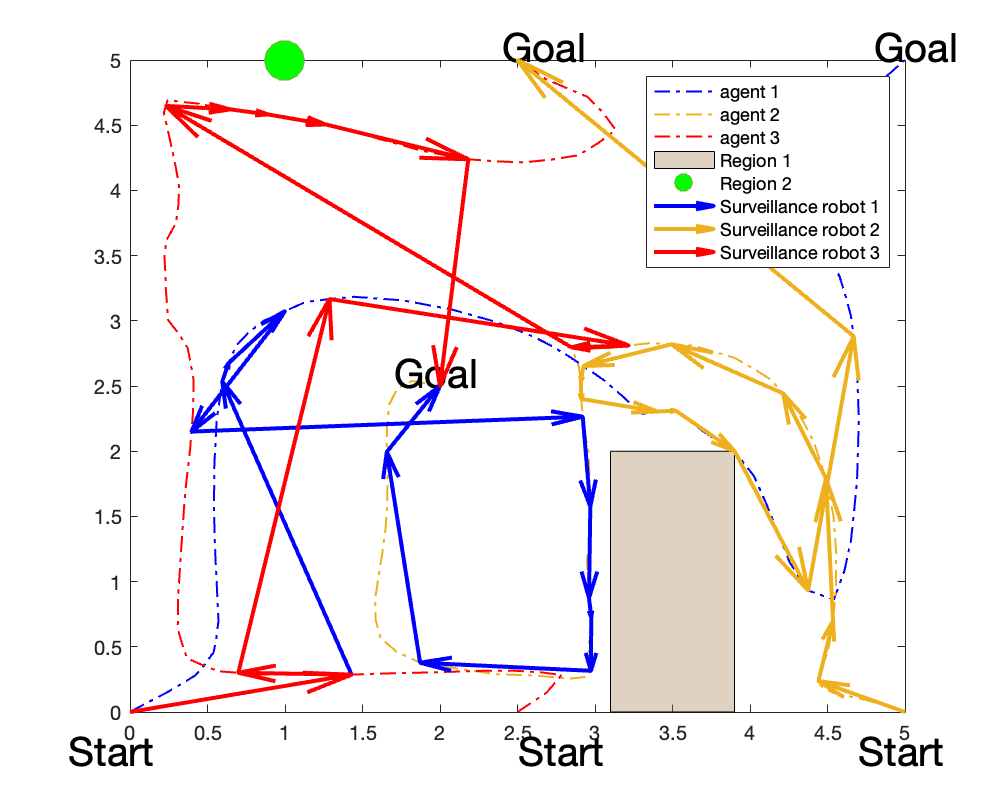}}
  \subfloat[Checkpoint graph and result flows \label{fig:previous_result_graph}]{\includegraphics[width=0.47\linewidth, trim = 10cm 3cm 7cm 3cm,clip,valign=b]{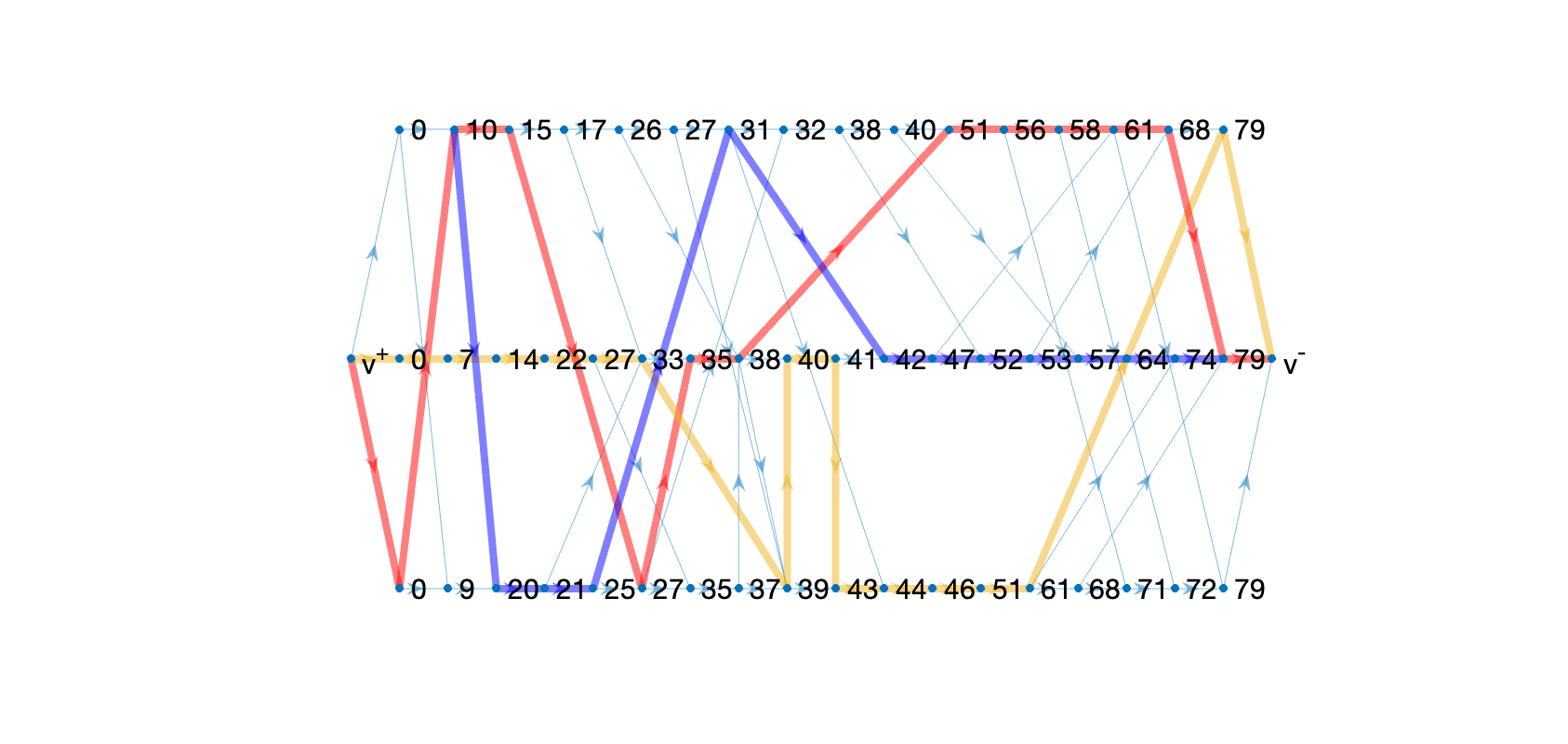}}
  \caption{(\ref{fig:previous_result_plan}) The cross-trajectory co-observation plan is shown as arrows on top of the original plan in \cref{fig:previous_result_plan}. (\ref{fig:previous_result_graph}) The checkpoint graph and resulting flows highlighted.}
  \label{fig:3-team-ctco}
\end{figure}

We first test the proposed method for the example application in \cref{fig:example-application,fig:ReachabilitySimulation}, using the same setup in \cref{sec:ADMM-simulation}. Using the parameters $w_{c}=10$, $w_{t}=1$ and $\rho = 0.01$, the result returns a total of $\cK=3$ surveillance robots with cross-trajectory plan shown in \cref{fig:3-team-ctco}. 
The flows derived from the solution of the optimization problem \eqref{eq:flow-coverage-problem} are highlighted in the graph \cref{fig:previous_result_graph}, where each horizontal line represents the original trajectory of a sub-team and the number on each vertex $v_{i}$ represents the corresponding time $t_{i}$. The planning result in the workspace is shown in \cref{fig:previous_result_plan} as dash-dotted arrows with the same color used for each flow in \cref{fig:previous_result_graph}. Compared with the result in \cref{fig:ReachabilitySimulation}, it is easy to see that there is a significant improvement in map coverage in \cref{fig:previous_result_plan}. At the same time, unsecured deviations to the forbidden regions of robots are secured through cross-trajectory observations. 

The problem shows no solution for $\cK\leq 2$. For cases $\cK=3$, the problem returns the optimal result as shown in \cref{fig:Cross-trajectory-result}. If we further increase $\cK>3$, we do not get a better result; instead, the planner will return four flows with the rest $\cK-3$ flows empty. 

We then test the proposed method for 4-team and 7-team cases with results shown in \cref{fig:Cross-trajectory-result}, where four trajectories are provided for a map exploration task with no security-related constraints (co-observation schedule and reachability). We have a $10m\times10m$ task space, three forbidden regions (rectangle regions in \cref{fig:Result-plan}), and robots with a max velocity of $0.5m/dt$. 

\begin{figure}
  \centering
  \subfloat[4-agent case\label{fig:Result-plan}]{\includegraphics[width=0.47\linewidth,trim = 1cm 1cm 1.5cm 0cm, clip,valign=c]{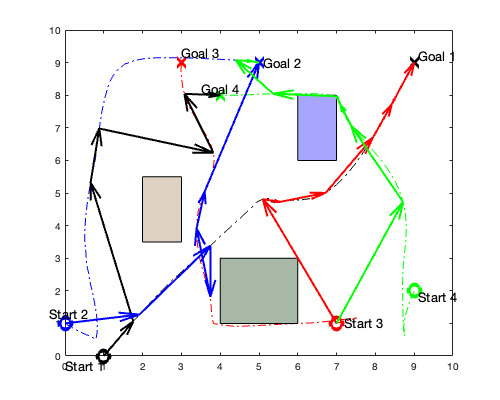}}
  \subfloat[7-agent case\label{fig:Result-plan-7-team}]{\includegraphics[width=0.47\linewidth,trim = 3.5cm 2cm 1.5cm 0cm, clip,valign=c]{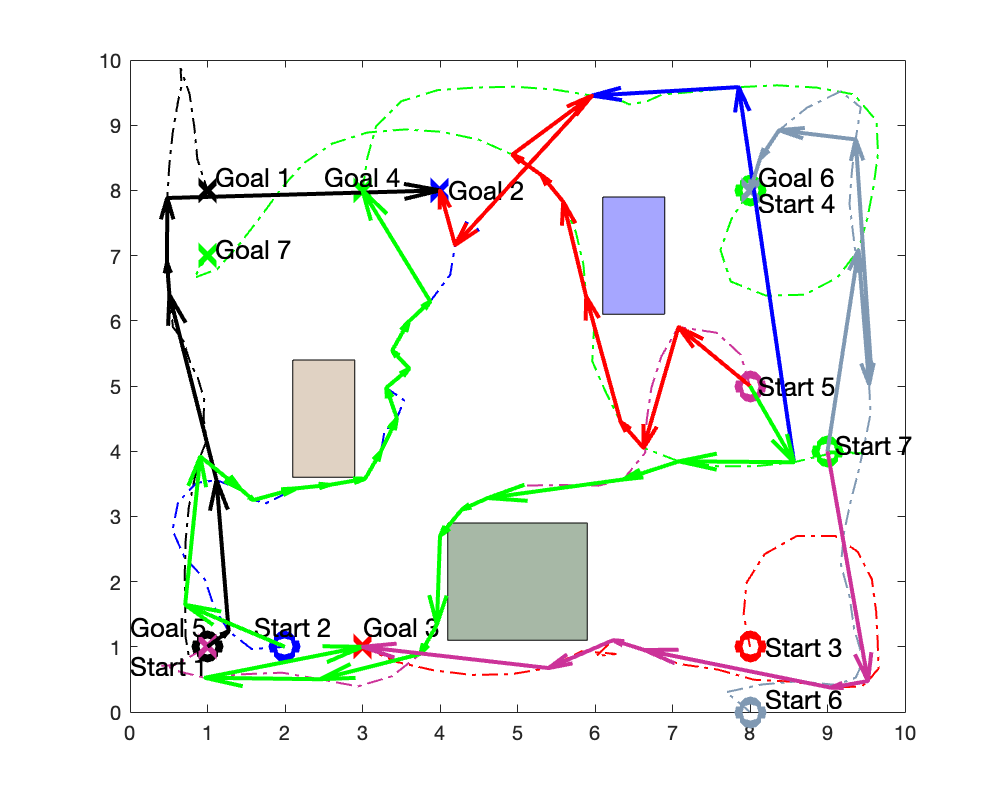}}
  \caption{(\ref{fig:Result-plan}) Cross-trajectory co-observation result of a 4 sub-teams task. (\ref{fig:Result-plan-7-team}) Cross-trajectory co-observation result of a 7 sub-teams task.}\label{fig:Cross-trajectory-result}
\end{figure}

\section{Summary}\label{sec:summary}

This paper introduces security measures for protecting Multi-Robot Systems (MRS) from plan-deviation attacks. We establish a mutual observation schedule to ensure that any deviations into forbidden regions break the schedule, triggering detection. Co-observation requirements and reachability analysis are incorporated into the trajectory planning phase, and formulated as constraints for the optimal problem. In scenarios where a secured plan is infeasible or enhanced task performance is necessary, we propose using redundant robots for cross-trajectory co-observations, offering the same security guarantees against plan-deviation attacks. However, the time-sensitive and proximity-dependent nature of co-observation necessitates a focus on collision avoidance during planning. Future work will integrate collision avoidance and explore dynamic duty assignments within sub-teams, increasing the challenge for potential attackers.

\appendix
\subsection{Transformation of a reachability ellipsoid to canonical frame}\label{apx:transformation}
The ellipse $\cE$ expressed in $\cF_\cE$ is given by $\Eframe{\cE}=\{\Eframe{q}\in\real{m}:d(q^\cE_1,\Eframe{q})+d(\Eframe{q},q^\cE_2)<2a\}$, with foci $q^\cE_1,q^\cE_2$ in $\cF_\cE$ defined as
$q^\cE_1=\bmat{c & 0 & 0}\transpose, q^\cE_2=\bmat{-c & 0 & 0}\transpose$, and semi-axis distance $c=\frac{\norm{q_2-q_1}}{2}$.
\begin{definition}
  The \emph{reachability ellipsoid $\cE$ in the canonical frame} is defined as the zero-level set of the quadratic function
  \begin{equation}\label{equ:standard-ellipse}
    \Eframe{E}(\Eframe{q}) = \Eframe{q}\transpose Q \Eframe{q} - 1
  \end{equation}
  where
  \begin{equation}
    Q = \diag(a^{-2},b^{-2},b^{-2}),
  \end{equation}
  and $b = \sqrt{a^2-c^2}$.
  The ellipse parameters $a$, $b$ represent the lengths of the major axes.
\end{definition}
\begin{lemma}
The original ellipse $\cE$ in $\cF$ can be expressed as the zero level set of the quadratic function
   \begin{equation}\label{eq:ellpsoid from canonical}
     \Fframe{E}(q) = (q-o_\cE)\transpose H\transpose Q H (q-o_\cE) - 1
     \end{equation}
\end{lemma}
\begin{proof}
    The claim follows by substituting  \eqref{eq:transformations} into \eqref{equ:standard-ellipse}, and from the definition of $R$ and $o$.
\end{proof}
\subsection{Householder rotations}\label{sec:householder}
We define a differentiable transformation to a canonical ellipse that is used in the derivation of the reachability constraints in \cref{sec:reachability,sec:ellipse-region-constraint}. This transformation includes a rotation derived from a modified version of Householder transformations \cite{householder1958unitary}. 
We call our version of the operator a \emph{Householder rotation}. In this section we derive Householder rotations and their differentials for the 3-D case; the 2-D case can be easily obtained by embedding it in the $z=0$ plane.
\begin{definition} Let $\nu_1$ and $\nu_2$ be two unitary vectors ($\norm{\nu_1}=\norm{\nu_2}=1$). Define the normalized vector $u = \frac{\nu_1+\nu_2}{\norm{\nu_1+\nu_2}}$,
the \emph{Householder rotation} $H(\nu_1,\nu_2)$ is defined as
  \begin{equation}\label{eq:H definition}
    H(\nu_1,\nu_2) = 2 u u\transpose-I.
  \end{equation}
\end{definition}
Here $H$ is a rotation mapping $\nu_1$ to $\nu_2$, as shown by the following.
\begin{proposition}\label{prop:HProperty}
 The matrix $H$ has the following properties:
  \begin{enumerate}
  \item\label{it:rotation} It is a rotation, i.e. (a) $H\transpose H=I$; (b) $\det(H)=1$
  \item\label{it:transformation} $\nu_2=H \nu_1$.
  \end{enumerate}
\end{proposition}
\begin{proof}
  For subclaim~\ref{it:rotation}a,  since $u\transpose u=1$:
  \begin{equation}
    H\transpose H=H^2=4uu\transpose u u\transpose - 4uu\transpose +I^2=I.
  \end{equation}

  For subclaim~\ref{it:rotation}b, let $U=\bmat{u & u^\bot_{1} & u^\bot_{2}}$, where $u^\bot_{1},u^\bot_{2}$ are two orthonormal vectors such that $I=UU\transpose=uu\transpose + u^\bot_{1}(u^\bot_{1})\transpose+u^\bot_{2}(u^\bot_{2})\transpose$; then, substituting $I$ in \eqref{eq:H definition}, we have that the eigenvalue decomposition of $H$ is given by
  \begin{equation}
    H=U\diag(1,-1,-1) U\transpose.
  \end{equation}
  Since the determinant of a matrix is equal to the product of the eigenvalues, $\det(H)=1$.

  For subclaim~\ref{it:transformation}, first note that $Hu=2uu\transpose u - u=u$.
  It follows that the sum of $\nu_1$ and $\nu_2$ is invariant under $H$:
  \begin{equation}\label{equ:H(u1+u2)}
    H(\nu_1+\nu_2)=H u \norm{\nu_1+\nu_2}
    = u\norm{\nu_1+\nu_2} =\nu_1+\nu_2,
  \end{equation}
  and that their difference is flipped under $H$:
  \begin{equation}\label{equ:H(u1-u2)}
    H(\nu_1-\nu_2) = 2uu\transpose (\nu_1-\nu_2) - (\nu_1-\nu_2)^2 = -(\nu_1-\nu_2).
  \end{equation}
  Combining (\ref{equ:H(u1+u2)}) and (\ref{equ:H(u1-u2)}) we obtain
  \begin{equation}
    H\nu_1 = \frac{1}{2}\bigl(H(\nu_1+\nu_2)+H(\nu_1-\nu_2)\bigr)
    = \nu_2
  \end{equation}
\end{proof}
We compute the differential of $H$ implicitly using its definition \eqref{equ:dt_to_dx}. We use the notation $\cross{v}:\real{3}\to\real{3\times 3}$ to denote the matrix representation of the cross product with the vector $v$, i.e.,
\begin{equation}
  \bmat{v_1\\v_2\\v_3}_\times= \bmat{0 & -v_3 &v_2\\v_3 & 0 & -v_1\\-v_2 & v_1 & 0},
\end{equation}
such that $[v]_\times w=v\times w$ for any $w\in\real{3}$. 
\begin{proposition}\label{prop:Hderivitive}
  Let $\nu_1(t)$ represent a parametric curve. Then
  \begin{equation}
    \dot{H}=H\cross{-2M\dot{\nu}_{\cF}},
  \end{equation}
  where the matrix $M\in\real{3\times3}$ is given by
  \begin{equation}
    M=[u]_{\times}  \frac{ \left( I - u u\transpose \right) \left( I- \nu_1 \nu_1\transpose \right)} {\norm{u'} \norm{\nu_1}}.
  \end{equation}
\end{proposition}
\begin{proof}
    From the definition of $H$ in \eqref{eq:H definition}, we have
    \begin{equation} \label{equ:H_dot original}
      \dot H =   2(\dot u u\transpose + u \dot u\transpose)
    \end{equation}

    Recall that $\dot{u}=\frac{1}{\norm{u'}}(I-uu\transpose)\dot{u}'$ (see, for instance, \cite{Tron:Arxiv14}), which implies $(I-uu\transpose)\dot{u}'=\dot{u}'$. It follows that $\dot{u}$ flips sign under the action of $H\transpose$:
    \begin{multline}
      H\transpose \dot u = (2u u\transpose-I)\frac{ \left( I - u u\transpose \right) } {\left\|u'\right\|} \dot u' \\
      =\frac{1}{\norm{u'}} (2u u\transpose-I-2u u\transpose u u\transpose+u u\transpose)\dot{u}'\\
      = -\frac{1}{\norm{u'}} (I-u u\transpose)\dot u'
      = -\dot u
    \end{multline}

    Inserting $HH\transpose=I$ in (\ref{equ:H_dot original}), we have
    \begin{equation}
      \begin{split}
        \dot H =  &  2H H\transpose(\dot u u\transpose + u \dot u\transpose)
        =  2 H (-\dot u u\transpose + u \dot u\transpose)\\
        =  &  -2H [[u]_{\times} \dot u]_{\times} \\
        = &  -2 H \left[ [u]_{\times}  \frac{ \left( I - u u\transpose \right) \left( I- \nu_1 \nu_1\transpose \right)} {\left\|u'\right\| \left\|\nu_1\right\|} \dot{\nu}_\cF\right]_{\times}\\
        = & -2 H \cross{M \dot{\nu}_\cF},
      \end{split}
    \end{equation}
    which is equivalent to the claim.
  \end{proof}

  \subsection{Proof of proposition \ref{prop:Ellipse2PointDiff}}\label{proof:Ellipse2PointDiff}
      To make the notation more compact, we will use $\partial_q f$ instead of $\partial_{\left[\begin{smallmatrix}q_1\\q_2\end{smallmatrix}\right]} f$ for the remainder of the proof.
    The differential of \eqref{equ:Point2EllipseProjection} can be represented as:
    \begin{equation}\label{equ:dPi_dt}
      \begin{split}
        \dot \pi_{p\cE} = &  \dot H^{-1}SH (q_{avoid} - o)  + H^{-1}\dot S H (q_{avoid} - o) \\
        &+ H^{-1}S \dot H (q_{avoid} - o) + (H^{-1}SH -I)\dot o\\
      \end{split}
    \end{equation}

    where
    \begin{equation}\label{equ:S_dot}
      \begin{split}
        \dot S  =& - S^2 (Q \dot s + s \dot Q)\\
        =& - S^2 (Q \partial_q s \dot q - \partial_b Q \partial_q b \dot q)
      \end{split}
    \end{equation}
    where
    \begin{equation}
      \partial_b Q = 2\frac{s}{b^3} \diag\{0,1,1\}
    \end{equation}
    To compute the derivative $\partial_q \pi$, we need the expression of $\partial_q b$, $\partial_q o$ and $\partial_q s$; the first two can be easily derived using the equations above:
    \begin{align}
      \partial_q b &= \frac{1}{4b}\bmat{q_1-q_2,q_2-q_1}\transpose\\
      \partial_q o &= \bmat{I/2,I/2}\transpose
    \end{align}

    In order to get $\partial_q s$, we use the fact that $F\bigl(s(q)\bigr)=0$ for all $q$; hence $F\bigl(\tilde{q}(t)\bigr)\equiv 0$, and $\partial_q F = 0$. We then have:

    \begin{equation}
      0=\dot F =  2q\transpose Q' \dot q + q\transpose \partial_s Q' q \dot s + q\transpose \partial_bQ' q \dot b
    \end{equation}
    where
    \begin{equation}
      \partial_s Q'= -\diag\left(\frac{2a^2}{(s+a^2)^3},\frac{2b^2}{(s+b^2)^3},\frac{2b^2}{(s+b^2)^3}\right).
    \end{equation}
    By moving term $\dot s$ to the left-hand side we can obtain:
    \begin{multline}\label{equ:s_dot}
      \dot s =  (q\transpose \partial_s Q' q)^{-1} (2q\transpose Q' \dot q + q\transpose \partial_b Q' q \dot b)\\
      =  (q\transpose \partial_s Q' q)^{-1} (-4q\transpose Q' H[U\dot q]_\times (q_{avoid} - o) \\
      -2q\transpose Q'H\dot o + q\transpose \partial_b Q' q \dot b)\\
      =  (q\transpose \partial_s Q' q)^{-1} (-4q\transpose Q' H[q_{avoid} - o]_\times U\dot q \\
      -2q\transpose Q'H\dot o + q\transpose \partial_b Q' q \dot b)
    \end{multline}

    The second term of equation (\ref{equ:dPi_dt}) turns into:
    \begin{multline}
      H^{-1}\dot S H (q_{avoid} - o)
      = - H^{-1} Q' q \dot s - s H^{-1} S^2 \partial_b Q q \dot b\\
      =   \big( (q\transpose \partial_s Q' q)^{-1} H^{-1} Q' q q\transpose  (4Q' H[q_{avoid} - o]_\times U  \\
      + 2Q' H\partial_q o - \partial_b Q' q q \partial_q b) -  s H^{-1} S^2 \partial_b Q q \partial_q b\big) \dot q\\
    \end{multline}

    Thus, equation \eqref{equ:dPi_dt} could be written as:
    \begin{multline}
      \dot \pi_{p\cE} = \big(-2 H [ SH(q_{avoid}-o)]_{\times}U   \\
      + \left( (q\transpose \partial_s Q' q)^{-1} H^{-1} Q' q q\transpose  (4Q' H[q_{avoid} - o]_\times U \right. \\
      \left.+ 2Q' H\partial_q o - \partial_b Q' q q \partial_q b) -  s H^{-1} S^2 \partial_b Q q \partial_q b\right) \\
      -2H^{-1} S H[q_{avoid}-o]_{\times} U  \\
      + (H^{-1}SH -I)\partial_q o \big) \dot q,
    \end{multline}
    from which the claim follows.
  
  \subsection{Proof of proposition \ref{prop:dpi_ne_dt}}\label{proof:dpi_ne_dt}
    We first need to derive $\dot{d}_\cE$ and $\dot{d}_{\cE t}$

    \begin{equation}
      \dot d_\cE = -n\transpose \partial_q o \dot q
    \end{equation}
    \begin{equation}\label{equ:dot det}
      \begin{split}
        \dot d_{\cE t} =&  (\dot n_\cE\transpose Q^{-1} n_\cE + n_\cE\transpose \dot Q^{-1} n_\cE + n_\cE\transpose Q^{-1} \dot n_\cE) /\sqrt{n_\cE\transpose Q^{-1} n_\cE}\\
        = & (\sqrt{n_\cE\transpose Q^{-1} n_\cE})^{-1}\left(-2n\transpose H[Q^{-1}n_\cE]_\times U \right.\\
        &\left. + n_\cE \transpose \partial_b Q^{-1} n_\cE \partial_q b  -2 n_\cE Q^{-1}H[n]_\times U \right) \dot q
      \end{split}
    \end{equation}
    Next, we need to derive  $\dot p_{t1}$, $\dot p_{t2} $ and $\dot p_\cL$. Since $p_\cL$ could be written as
    \begin{equation}
      p_\cL = \frac{d_\cE Q^{-1} n_\cE}{ d_{\cE t} ^2},
    \end{equation}
    we have
    \begin{multline}\label{equ:dpl_dt}
        \dot p_\cL =  \left( (-\frac{d_{\cE t} n\transpose \partial_q o -2d_\cE \partial_q d_{\cE t}}{d_{\cE t}^3} )Q^{-1}n_\cE \right.\\
        \left .+ \frac{d_\cE\partial_b Q^{-1} n_\cE \partial_q b -  2d_\cE Q^{-1}H[n]_\times U}{d_{\cE t}^2}\right) \dot q
    \end{multline}
    \begin{multline}\label{equ:dplt_dt}
        \dot p_{1} =  \left(-\frac{Q^{-1}n_\cE \partial_q d_{\cE t} }{d_{\cE t}^2} \right.\\
         \left. + \frac{\partial_b Q^{-1}n_\cE \partial_q b -  2Q^{-1}H[n]_\times U }{d_{\cE t}} \right) \dot q
    \end{multline}
    subtracting $\dot q$ from (\ref{equ:dpl_dt}) and (\ref{equ:dplt_dt}), we can derive the result shown in (\ref{equ:ProjectPoint})

{\small
\bibliographystyle{IEEEtran}
\bibliography{ADMM_planning,ACC,tron,ziqi,reachability}
}

\end{document}